\documentclass[letterpaper, 10 pt, conference]{ieeeconf}
\IEEEoverridecommandlockouts
\usepackage{url}
\usepackage{microtype}
\usepackage{booktabs} 
\usepackage[utf8]{inputenc} 
\usepackage{amsfonts}       
\usepackage{dsfont}
\usepackage{nicefrac}       
\usepackage{microtype}      
\usepackage{array}
\usepackage{xifthen}
\usepackage{graphicx} 
\usepackage{subfigure}
\usepackage[short]{optidef}
\usepackage{amsmath} 
\usepackage{amssymb}
\usepackage{mathtools} 
\usepackage{mathrsfs}   
\usepackage[strict]{changepage}
\usepackage{framed}
\usepackage{paralist}
\usepackage{tikz}
\usetikzlibrary{calc}
\usepackage{pgfplots}
\pgfplotsset{compat=1.11}
\usepgfplotslibrary{fillbetween}
\newcommand\Rademacher[1]{\mathfrak{R}_{#1}}
\newcommand\EmpRademacher[1]{\widehat{\mathfrak{R}}_{#1}}

\newcommand\E[2]{
  {\mathbb{E}_{#1}\big[{#2}\big]}}
\newcommand\Eemp[2]{
  {\mathbb{\widehat{E}}_{#1}\big[{#2}\big]}}
\newcommand\R[3]{
 \ifthenelse{\isempty{#3}}%
  {\mathcal{R}_{#1}\big[{#2}\big]}
  {\mathcal{R}_{#1}^{#3}\big[{#2}\big]}
  }
\newcommand\Remp[2]{
  {\widehat{\mathcal{R}}_{#1}\big[{#2}\big]}}

\newcommand\Opt[4]{
  \ifthenelse{\isempty{#2}}%
  {\mathrm{{#1}}}
  {\underset{\substack{#2}}{\mathrm{{#1}}}}
  \;
  \ifthenelse{\isempty{#4}}%
  {{#3}}
  {\Big\{ {#3}\  \Big| \  {#4} \Big\}}
}
\usepackage{multirow}

\DeclarePairedDelimiter{\norm}{\lVert}{\rVert} 
\newcolumntype{C}[1]{>{\centering\arraybackslash}p{#1}}

\newtheorem{definition}{Definition}[section]

\newtheorem{theorem}{Theorem}[section]
\newtheorem{lemma}{Lemma}[section]
\newtheorem{corollary}{Corollary}[section]
\newtheorem{remark}{Remark}[section]
\graphicspath{{./pics/}}
\allowdisplaybreaks

\title{\LARGE \bf Robust optimization for adversarial learning with finite sample complexity guarantees} 

\author{\normalsize Andr\'e Bertolace$^1$, Konstantinos Gatsis$^2$, Kostas Margellos$^1$
  \thanks{$^1$Department of Engineering Science, University of Oxford, Oxford OX1 3PJ, U.K.
    E-mail:  andre.bertolace@eng.ox.ac.uk, 
    kostas.margellos@eng.ox.ac.uk}
  \thanks{$^2$Department of Engineering Science, University of Southampton, Southampton, UK.
    E-mail: k.gatsis@soton.ac.uk}}

\date{\today} 

\begin{document}

\maketitle 

\begin{abstract}
Decision making and learning in the presence of uncertainty has attracted significant attention in view of the increasing need to achieve robust and reliable operations. In the case where uncertainty stems from the presence of adversarial attacks this need is becoming more prominent. In this paper we focus on linear and nonlinear classification problems and 
propose a novel adversarial training method for robust classifiers, inspired by Support Vector Machine (SVM) margins. We view robustness under a data driven lens, and derive finite sample complexity bounds for both linear and non-linear classifiers in binary and multi-class scenarios. Notably, our bounds match natural classifiers' complexity. Our algorithm minimizes a worst-case surrogate loss using Linear Programming (LP) and Second Order Cone Programming (SOCP) for linear and non-linear models. Numerical experiments on the benchmark MNIST and CIFAR10 datasets show our approach's comparable performance to state-of-the-art methods, without needing adversarial examples during training. Our work offers a comprehensive framework for enhancing binary linear and non-linear classifier robustness, embedding robustness in learning under the presence of adversaries.
\end{abstract}

\section{Introduction}
Decision making and learning in the presence of uncertainty have considered significant attention in recent years, in particular due to the advancements in the machine learning literature that have opened the road for data driven considerations. However, adversaries may manipulate data to compromise model outcomes \cite{NileshDomingos2004}, and as such call for robust solutions. Adversarial attacks, particularly in neural networks \cite{NileshDomingos2004, Szegedy2014, Goodfellow2015, Battista2018, Anirban2018, Ram2019, Yuan2019, Battista2018, BaiTao2021}, have become a significant concern for safety-critical applications like autonomous driving \cite{Kukain2016, Yuan2019}. Various attack methods such as Limited-Memory BFGS \cite{Szegedy2014}, Fast Gradient Sign Method \cite{Goodfellow2015}, and Projected Gradient Descent \cite{Madry2019} have been explored. To address these, research has focused on developing defense mechanisms like defensive distillation \cite{Papernot2016} and feature squeezing \cite{Battista2018}, yet these defenses often lack comprehensive guarantees, highlighting the need for further research towards unifying attacks and defense mechanisms through robust optimization frameworks \cite{Madry2019}. 

The relationship between robust optimization and adversarial machine learning is notably strong. Recent studies have introduced a probabilistic framework that effectively balances average and worst-case scenarios \cite{Robey2022}. This framework also has ties to research on the so called scenario approach \cite{Book_CampiGaratti2018}, particularly in its applications to Support Vector Machines (SVM) \cite{CampiGaratti2008, GarattiCampi2019, CampiGaratti2021} and learning in general \cite{Margellos2014, Campi2023}. Furthermore, the Lipschitz constant for deep neural networks (DNNs) has emerged as a valuable tool in  certifying the robustness of classifiers and analyzing the stability of systems equipped with reinforcement learning controllers \cite{Pauli2022, Mahyar2019}.

\subsection{Our methodology and contribution}
In this paper, we present a novel adversarial training method inspired by SVM  \cite{CortesVapnik1995} margin concepts for binary and multi-class linear and non-linear classifiers. Unlike prior approaches, we analyze manipulations through classifier margins. Our contributions:

\begin{compactitem}
\item Establishing sample complexity bounds within a probably approximately correct (PAC)-learning framework for robust classifiers, leveraging input and parameter space norms. Notably, linear classifiers' sample complexity scales as $m \sim \mathcal{O}(\frac{1}{\epsilon^2}\log{\frac{2}{\delta}})$, where $\epsilon$ is a prespecified classsification accuracy level, and $\delta$ denotes the confidence.
\item Introducing a data-driven optimization-based adversarial training procedure using linear programming (LP) for linear models and second-order cone programming (SOCP) for non-linear ones.
\item Validating our approach on MNIST and CIFAR10 datasets, typically used as benchmarks in classification studies, demonstrating comparable performance to state-of-the-art methods achieving (probabilistic) robustness without the need to generate adversarial examples during training, thus reducing computational effort. 
\end{compactitem}

Our work offers a comprehensive framework for robustness enhancement, eliminating the need for fine-tuning penalization coefficients and specific adversarial examples.

\subsection{Related work}
Our sample complexity bounds match those in\cite{Attias2019} but without assuming adversary tampering per input. They achieve $\mathcal{O}(\frac{1}{\epsilon^2}(k \log(k) \text{VC}(\mathcal{H}) + \log\frac{1}{\delta}))$ using a zero-sum game framework extended to multi-class and real-valued cases. Another work \cite{DongBarlett2019} shows that adversarial Rademacher complexity for binary linear classifiers is never smaller than natural Rademacher complexity, consistent with our findings. Another result, \cite{Bhattacharjee2021} achieved $\mathcal{O}(\frac{1}{m})$ expected standard loss for linear classifiers under separable data assumptions, whereas our approach adds flexibility by accommodating real-world datasets often not meeting separability assumptions. And \cite{Ashtiani2023} studied tolerant adversarial PAC-learning with a larger perturbation radius, deriving sample complexity bounds based on VC-dimension.

On the algorithmic side, \cite{Madry2019} unified attacks and defenses through robust optimization, shaping adversarial machine learning. TRADES \cite{TRADES_ZhangElGhaoui2019} proposed a method to trade adversarial robustness for accuracy by leveraging natural error and boundary error decomposition. SMART \cite{SMART_WangZou2020} introduces a technique considering misclassification and differentiating between misclassified and correctly classified examples during training.

\section{Learning in the presence of an adversary}

Let $\big(\Omega, \mathscr{F}, \mathcal{P} \big)$ be a probability space, and $\omega: \big( \Omega, \mathscr{F} \big) \rightarrow \big( Z, \mathscr{Z} \big)$ be a measurable mapping. Note that since $\omega$ is measurable we can define the image probability measure of $\mathcal{P}$ through $\omega$, defined over the Borel $\sigma$-algebra on $Z$, as
\[
 \mathbb{P}(z) = \big(\mathcal{P} \circ \omega^{-1}\big) (z) = \mathcal{P}\big( \omega^{-1}(z) \big), \forall z
  \in \mathscr{Z}.
  \label{eqn:Pomega}
\]
In addition, let $\mathcal{H}$ be a class of hypotheses or models. Each hypothesis $h \in \mathcal{H}$ is a function mapping $ X \rightarrow Y$, where $X$ represents the domain of features and $Y$ the domain of response variables. In classification context such a hypothesis can be simply termed as classifier. 

We are concerned with the learning problem, in which the learner aims at finding the best hypothesis $h$ that minimizes a certain risk, i.e., $\Opt{inf}{h \in \mathcal{H}}{\R{\mathbb{P}}{\ell(h(x), y)}{}}{}$,
where $\ell: Y \times Y \rightarrow \mathbb{R}_+$ is a loss function and, for a fixed probability measure $\mathbb{P}$ and $\R{\mathbb{P}}{\cdot}{}$ is a functional quantifying the risk. Generally the risk is taken to be the expected value associated with $\mathbb{P}$, leading to, finding $h$ that minimizes 
\begin{equation}
\Opt{inf}{h \in \mathcal{H}}{\E{\mathbb{P}}{\ell(h(x), y)}}{}\label{Unc}.
\end{equation}

%

\subsection{Adversarial attacks and related approaches}
Considerations of adversarial attacks after learning involve a common modeling assumption \cite{Cullina2018, FeigeMansour2015, Madry2019,  LeeRaginsky2018, TuZhang2019}  which dictates that an adversary can manipulate data features within a certain vicinity of the original example, $x$. Formally, given a data perturbation, or adversarial power, $\xi$, the manipulated sample is denoted by $\tilde{x} \in \mathcal{B}_\xi(x)$, where,
\begin{equation}
\label{equation:adversarial_region}
\mathcal{B}_\xi(x) := \Big\{\tilde{x}:~ \norm{\tilde{x}-x} \le \xi \Big\},
\end{equation}
and the choice of the norm used to measure the distance can be arbitrary. For more details on state-of-the-art attacks, such as the Fast Gradient Sign Method (FGSM), Projected Gradient Descent (PGD), Carlini and Wagner (CW), and Deep Fool, please refer to Appendix \ref{appendix:attacks} for more information on these attacks.

An effective approach for learning models to defend against adversarial examples is a procedure called adversarial training \cite{Goodfellow2015, Madry2019}. The core idea is to expose the model, during the training process, to adversarial examples crafted to intentionally deceive it.  As a result, adversarially trained models learn to better defend against attacks, leading to increased predictability and reliability during inference. The concept emerged by studying the adversarial robustness of neural networks through the lens of robust optimization \cite{Madry2019}. More precisely, the authors examined the following parameterized $\min-\max$ problem,
\begin{equation}
\label{equation:min-max}
  \Opt{min}
  {\theta}
  {\E{\mathbb{P}}{\Opt{max}{\tilde{x} \in \mathcal{B}_\xi(x)}{\ell(\theta, \tilde{x}, y)}{}}}
  {}.
\end{equation} 

This formulation enabled the authors to cast both attacks and defenses within a common theoretical framework, naturally encapsulating most prior work on adversarial examples. Specifically, to reliably train models that are robust to adversarial attacks, they propose the adversarial empirical risk minimization (AERM) paradigm, where the learner does not know the distribution $\mathbb{P}$ but has access to $m$ independently and identically distributed examples $S = \big(z_1, \ldots, z_m\big) \in Z^m$. Setting $Z = X \times Y$ and $\omega = (x, y)$; then each point $z_i = (x_i, y_i)$ is sampled from the fixed but possibly unknown distribution $\mathbb{P}$. Note that $S$ induces a probability over $Z^m$ which we will denote by the product measure $\mathbb{P}^m$.
\begin{equation}
\label{equation:min-max-empirical}
  \Opt{min}
  {\theta}
  {\frac{1}{m}\sum_{i=1}^m \Opt{max}{\tilde{x}_i \in \mathcal{B}_\xi(x_i)}{\ell(\theta, \tilde{x}_i, y_i)}{}}
  {}.
\end{equation}



This approach has a clear impact on the result of the optimization problem. Take as an example the linear regression problem where, given a set of samples $S$, the traditional (non-adversarial) learner proceeds by deciding on the values of $a, b$ by means of the following ERM procedure,
\[
  \Opt{min}{a,b \in R^d}{\frac{1}{2m}\sum_{i=1}^m{\norm{y_i - (a^Tx_i + b)}_2^2}}{}.
\]
On the contrary, the AERM learner would formulate the following robust optimization counterpart
\[
  \Opt{min}
  {a, b \in \mathbb{R}^d}
  {\frac{1}{m}\sum_{i=1}^m \norm{y_i - (a^Tx_i + b + \xi \norm{a}_*)}_2^2}
  {},
\]
where $\|\cdot\|_*$ denotes the dual norm, that emanates through the reformulation of a $\min-\max$ robust program as in \eqref{equation:min-max}.
We refer to Appendix \ref{appendix:examples} for more details on the linear (Appendix \ref{appendix:linear_regression}) and logistic (Appendix \ref{appendix:logistic_regression}) regression problems.

In the upcoming sections, we will show that our approach distinguishes itself from AERM as we remove the need to solve the inner maximization problem. Instead, we focus on the robust counterpart of Equation \ref{equation:min-max-empirical}, whose theoretical PAC-learning guarantee is exposed in Section \ref{section:sample_complexit_bounds}, with the resulting optimization algorithms detailed in Section \ref{subsection:algorithms}.

Also, consider a gradient ascent step towards solving the inner maximization problem in \eqref{equation:min-max}.
\[
\tilde{x} = x + \xi \cdot \text{sign}(\nabla_x \ell(h_\theta(x, y)).
\]
This update results in adversarial examples and constitutes an attack. In particular, such an attack is considered for an $\ell_\infty$-bounded adversary in \cite{Goodfellow2015}, and is referred to as the FGSM attack. We will employ such an attack for the numerical results presented in the sequel.

\subsection{Proposed approach: margin-inspired adversarial training}

Typically, for a binary classifier, the learner minimizes the natural classification error, with the loss function $\ell(h(x), y) = \mathds{1}_{\{x\in X:~y \cdot h(x)<0\}}$, where the natural risk is 
\begin{align}
\R{\mathrm{nat}}{h}{} &= \E{\mathbb{P}}{\mathds{1}_{\{x\in X:~y \cdot h(x) \leq 0\}}} \nonumber \\
&= \int_\Omega \mathds{1}_{\{x\in X:~y \cdot h(x) \leq 0\}} d\mathbb{P}.
\end{align}

Drawing inspiration from SVM's margin theory (see Appendix \ref{appendix:svm_and_margin_theory}), we explore the concept of confidence margin for binary classification tasks. Given a real-valued function $h$ that operates on a data point $x$ labeled with $y$, the confidence margin is defined as $h^\prime(x,y) = y \cdot h(x)$. Thus, a correct classification by $h$ occurs when $h^\prime(x,y) > 0$, signifying that $x$ is classified accurately. Notably, $\big| h(x) \big|$ can be interpreted as the level of confidence in the prediction made by $h$.

Recall that in the presence of an adversary, the classifier might encounter $\tilde{x} \in \mathcal{B}_\xi(x)$ as defined in \eqref{equation:adversarial_region}. Although this constraint is imposed on the feature space $X$, our goal is to ensure PAC learnability for the class of functions $\mathcal{H}$. To achieve this, we limit ourselves to working with well-behaved functions. Specifically, we consider only functions $h$ that are Lipschitz continuous, which means that there exists a constant $B$ such that for all $x_0, x_1 \in X$, the following inequality holds,
\begin{equation}
\label{equation:lipschitz_constant}
\norm{h(x_1) - h(x_0)} \leq B \norm{x_1 - x_0}.
\end{equation}

Similar to the definition of $\tilde{x} \in \mathcal{B}_\xi(x)$, we can define the neighborhood of the decision boundary of $h$, namely $DB(h)$, as \cite{TRADES_ZhangElGhaoui2019}, i.e.,
\begin{equation}
    \label{equation:adversarial_region_h}    
  \mathcal{B}_\xi(DB(h)) = \{x \in X:\exists \Tilde{x} \in \mathcal{B}_\xi(x) | h(x) \cdot h(\Tilde{x}) \leq 0 \}.
\end{equation}

This motivates the definition of the robust and the boundary classification errors respectively as
\begin{align}
    \label{equation:R_rob_xi}
  \R{\mathrm{rob}}{h}{\xi} &= \E{\mathbb{P}}{\mathds{1}_{\{x \in X:~ \exists \Tilde{x} \in \mathcal{B}_\xi(x) | y \cdot h(\Tilde{x}) \leq 0\}}}, \\
  \label{equation:R_bdy_xi}
  \R{\mathrm{bdy}}{h}{\xi} &= \E{\mathbb{P}}{\mathds{1}_{\{x \in \mathcal{B}_\xi(DB(h)) | y \cdot h(x) > 0\}}}.
\end{align}

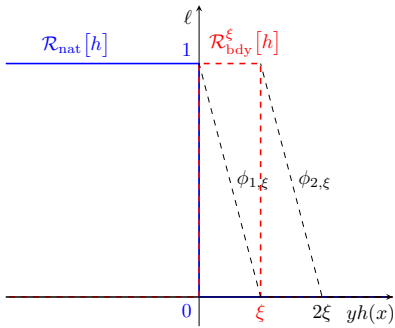
\begin{figure}[h]
  \centering
\begin{tikzpicture}[scale=0.75]
 \begin{axis}[axis lines=center, ymin = -0.125,ymax = 1.25, domain = -pi:pi, ticks=none]
 \node[below] at (2.79,0) {$yh(x)$};
 \node[left] at (0, 1.2) {$\ell$};
 \draw[thick, blue] (pi,0) -- (0,0) node[below left]{$0$} -- (0,1) node[above left]{$1$} -- (-pi,1);
 \node[above] at (-2, 1) {\textcolor{blue}{$\R{\mathrm{nat}}{h}{}$}};
 \draw[dashed, thick, red] (pi,0) -- (1,0) node[below]{$\xi$} -- (1,1) -- (0,1) -- (0,0) -- (-pi, 0);
 \node[above] at (0.75, 1) {\textcolor{red}{$\R{\mathrm{bdy}}{h}{\xi}$}};
 \draw[dashed] (pi,0) -- (1,0) -- (0,1);
 \node[right] at (0.5, 0.5) {$\phi_{1,\xi}$};
 \draw[dashed] (pi,0) -- (2,0) node[below]{$2\xi$} -- (1,1) ;
 \node[right] at (1.5, 0.5) {$\phi_{2,\xi}$};
 \addplot [mark=thick] {0};
 \end{axis}
\end{tikzpicture}
\caption{Graphical representation of the decision boundary and errors: natural error (blue), boundary error (dashed-red), and robust error (dashed-black).}
\label{fig:graph}
\end{figure}

The boundary error measures the probability of points correctly classified but near the boundary, which might be misclassified by a powerful adversary.
As a result of these definitions, the robust classification error can be decomposed into the natural classification error and the boundary classification error \cite{TRADES_ZhangElGhaoui2019},
\begin{equation}
  \R{\mathrm{rob}}{h}{\xi} = \R{\mathrm{nat}}{h}{} + \R{\mathrm{bdy}}{h}{\xi}.
\end{equation}
In other words, by inspection of Fig. \ref{fig:graph}, 
$\R{\mathrm{bdy}}{h}{\xi}$ constitutes a margin modification with respect to 
$\R{\mathrm{nat}}{h}{}$, to embed robustness towards example perturbations up to level $\xi$.
Let $\phi_{\lambda, \xi}$ be the surrogate loss, as shown graphically in Fig. \ref{fig:graph},
\begin{equation}
\phi_{\lambda, \xi}(z) = \Big(\lambda - \frac{z}{\xi} \Big)_+ = \max \Big(0, \lambda - \frac{z}{\xi} \Big).
\end{equation}
Due to the dominance conditions induced by these surrogate loss functions, we then have that
\begin{align*}
  \R{\mathrm{nat}}{h}{} \leq \E{}{\phi_{1,\xi}(y \cdot h(x))} 
  &\leq \R{\mathrm{rob}}{h}{\xi} \nonumber \\ &\leq \E{}{\phi_{2, \xi}(y \cdot h(x))}.
\end{align*}

Even though representing these sets in term of $\xi$ is the most natural, it is easier to work directly with another constant $\zeta =  B \xi$, where $B$ is the Lipschitz constant of $h$ as defined in \eqref{equation:lipschitz_constant}. Thus, we can introduce the following set,
\begin{equation}
  \mathcal{B}_\zeta(x) = \{\Tilde{x}  | \norm{h(\Tilde{x}) - h(x)} \leq \zeta \}.
\end{equation}

Inspired by the definition of $\mathcal{B}_\xi(DB(h))$, in \eqref{equation:adversarial_region_h}, we can also define,
\begin{equation}
  \mathcal{B}_\zeta(DB(h)) = \{x \in X | \exists \Tilde{x} \in \mathcal{B}_\zeta(x) \big| h(x) \cdot h(\Tilde{x}) \leq 0 \}.
\end{equation}
It is worth mentioning that both sets $\mathcal{B}_\zeta(x)$ and $\mathcal{B}_\zeta(DB(h))$ contain the sets $\mathcal{B}_\xi(x)$ and $\mathcal{B}_\xi(DB(h))$ respectively.

This leads to a reformulation of the robust classification error,
\begin{align}
  \label{equation:R_rob_zeta}
  \R{\mathrm{rob}}{h}{\zeta} &= \E{\mathbb{P}}{\mathds{1}_{\{x:~ \exists h(\Tilde{x}) \in \mathcal{B}_\zeta(x) | y \cdot h(\Tilde{x}) \leq 0\}}}.
\end{align}
We aim at designing PAC bounds for $\R{\mathrm{rob}}{h}{\zeta}$, that as a result will constitute probabilistic classification statements. In the next section we show how to determine such bounds first for binary classifiers, and subsequently for multi-class ones.

\section{Main results}

\subsection{Sample complexity bounds}
\label{section:sample_complexit_bounds}
\subsubsection{Binary classifiers}
\label{subsection:binary_classifier}

Theorem \ref{theorem:main_theorem} below constitutes our main theoretical result. It provides PAC learning bounds on the robust classification error in \eqref{equation:R_rob_zeta}. In particular, we provide complexity bounds for the sample size $m$, showing that the worst-case surrogate loss adversarial training method results in a learned classifier that can achieve a given classification accuracy of level $\epsilon$, with confidence at least $1-\delta$, for given $\epsilon, \delta \in (0,1)$. The obtained sample size bounds provide explicit expressions for $m$ as a function of $\epsilon$, $\delta$, the adversarial power $\zeta$ and the complexity of the class of hypotheses $\mathcal{H}$, represented by the Rademacher complexity, $\Rademacher{m}(\mathcal{H})$ (Definition \ref{definition:Rademacher} in Appendix \ref{appendix:definitions_and_theorems}), that are of the same complexity with their non-adversarial counterparts up to the level of a constant.
We show how to construct classifiers that enjoy such PAC properties in the following subsection.

Note that subsequent results involve considering an arbitrary $\zeta$; this is effectively equivalent to fixing an arbitrary $\xi$.

\begin{theorem}
\label{theorem:main_theorem}
\emph{Binary classifier.} Consider the hypothesis class $\mathcal{H}$ of Lipschitz continuous functions. Fix any $\zeta>0$. We then have that, with probability at least $1-\delta$, for any $h \in \mathcal{H}$, 
  \begin{equation*}
    \R{\mathrm{rob}}{h}{\zeta} \leq \frac{1}{m} \sum_{i=1}^m \phi_{2, \zeta}(y_i \cdot h(x_i)) + \frac{2}{\zeta}\Rademacher{m}(\mathcal{H}) + \sqrt{\frac{\log{\frac{1}{\delta}}}{2m}}.
  \end{equation*}

The aforementioned bound holds uniformly, i.e., for $\gamma > 1$ and for any fixed $r>0$, with probability at least $1-\delta$, for all $\zeta \in ]0, r]$, and for any $h \in \mathcal{H}$,
  \begin{align*}
    \R{\mathrm{rob}}{h}{\zeta} \leq& \frac{1}{m} \sum_{i=1}^m \phi_{2, \zeta}(y_i \cdot h(x_i)) + \frac{2 \gamma}{\zeta}\Rademacher{m}(\mathcal{H}) \\
    &+ \sqrt{\frac{\log \log_\gamma \frac{\gamma r}{\zeta}}{m}} + \sqrt{\frac{\log{\frac{2}{\delta}}}{2m}}.
  \end{align*}
\end{theorem}

\begin{proof}
Given the definitions \eqref{equation:R_rob_zeta} we use standard inequalities in the statistical learning theoretic literature and Talagrand's lemma to bound the Rademacher complexity of the loss functions by the complexity of the hypothesis class $\mathcal{H}$. In particular, we show the result holds for all $\zeta \in ]0,r]$, by appropriately choosing series of $\zeta_k,\ \epsilon_k$ that converge uniformly. We refer to Appendix \ref{appendix:proofs:main_theorem} for a complete proof.
\end{proof}

\begin{remark} \emph{On the uniformity of the convergence statement.}
 Uniform convergence not only strengthens the convergence notion but also empowers the learner to make informed decisions about the model's performance against various adversaries, making it a valuable tool in practical machine learning applications. For instance, the first inequality in Theorem \ref{theorem:main_theorem} allows the learner to calculate the required sample size to achieve a desired level of accuracy and confidence against one adversary. However, when confronted with a stronger adversary, the learner is not be able to provide the same guarantees regarding the model's accuracy or confidence level. Instead, considering the uniform convergence case, the learner is able to ensure accuracy and confidence for a range of adversaries just after training. In other words, once the model has been trained, the learner can confidently assert its performance regarding accuracy and confidence against a range of adversaries.


When comparing both inequalities in Theorem \ref{theorem:main_theorem}, we notice that the statement with uniform convergence has only a small impact the sample complexity, as the order of the sample complexity remains unchanged, differing only by some constants.
\end{remark}

\begin{remark} \emph{Price of robustness.} 
For the binary classifier, in a standard learning process without an adversary, by Theorem \ref{theorem:rademacher_bound_binary_classifier} we have that with probability at least $1-\delta$,
\[
  \R{\mathrm{nat}}{h}{} \leq \frac{1}{m} \sum_{i=1}^m \phi_{2, \zeta}(y_i \cdot h(x_i)) + \Rademacher{m}(\mathcal{H}) + \sqrt{\frac{\log{\frac{1}{\delta}}}{2m}}.
\]

For the sake of simplicity, consider a strong adversary and choose $\gamma = \zeta=r > 1$. In this case, the inequality stated in Theorem \ref{theorem:main_theorem} takes a simpler form,
\[
    \R{\mathrm{rob}}{h}{\zeta} \leq \frac{1}{m} \sum_{i=1}^m \phi_{2, \zeta}(y_i \cdot h(x_i)) + 2\Rademacher{m}(\mathcal{H}) + \sqrt{\frac{\log{\frac{2}{\delta}}}{2m}}.
\]
This indicates that the sample complexity has the same order as the natural training procedure, only differing by constants influencing the Rademacher complexity.
\end{remark}

\begin{remark} \emph{Effect of the adversarial power.}
As the parameter $\zeta$ increases, the influence of the Rademacher complexity of the hypothesis class on the sample complexity diminishes. However, a higher $\zeta$ also leads to a more loose approximation of the desired loss by the surrogate loss as seen in Fig \ref{fig:graph}, indicating a deterioration in the quality of the approximation. In this case, it is intuitive that fewer observations are required to satisfy a more conservative inequality.
Conversely, when $\zeta$ is small, we witness the opposite effect, the surrogate loss approximation becomes tighter, at the cost of a higher dependency on the Rademacher complexity of the hypothesis class.
\end{remark}

In the following results we choose to state the uniform convergence version of the theorem. A non-uniform version is easily achievable by omitting the extra term in the equations.

\subsubsection{Linear binary classifiers with bounded inputs} \label{case:class_linear}

We now specialize attention to the case  where $\mathcal{H}$ is the class of affine functions $a^Tx+b$ in $\mathbb{R}^d$. 
For non-homogeneous half spaces in $\mathbb{R}^d$, 
The Rademacher complexity $\Rademacher{m}(\mathcal{H})$ for such classes can be bounded by a function of the VC dimension, which is in turn bounded by $d+1$ (see Theorem \ref{Theorem:VCdim_binary}),
\begin{align*}
  \R{\mathrm{rob}}{h}{\zeta} \leq & \frac{1}{m} \sum_{i=1}^m \phi_{2, \zeta}(y_i \cdot h(x_i)) + \sqrt{\frac{\log \log_\gamma \frac{\gamma r}{\zeta}}{m}} \\
  & + \frac{2 \gamma}{\zeta} \sqrt{\frac{2(d+1)\log \frac{em}{d+1}}{m} } + \sqrt{\frac{\log{\frac{2}{\delta}}}{2m}}.
\end{align*}
Compact spaces reduce the impact of dimensionality $d$ on sample complexity, as demonstrated in the following lemma and corollary. They remove the reliance on the Rademacher complexity of class $\mathcal{H}$ by utilizing norm and dual-norm bounds in the input and parameter space. While this may appear as a stringent restriction, it is natural in image classification problems as pixels have maximum attainable values.

\begin{lemma} 
\label{lemma:linear_classifier_rademacher_bound}
Let $\mathcal{H}$ the hypothesis class of affine functions as defined above. Assume that, for all $x \in X$, $\norm{x^\prime} \leq u$, with $x^\prime = [x^T, 1]^T$, and that $\norm{w}_* \leq v$, where $w = [a^T, b]^T$. We then have that $\Rademacher{m} (\mathcal{H}) \leq \sqrt{\frac{u^2 v^2}{m}}$.
\end{lemma}
\begin{proof}
The proof bounds the Rademacher complexity through the use of the dual norm, which is typically employed to reformulate robust optimization programs. The complete proof is available at Appendix \ref{appendix:proofs:main_theorem:linear_classifier_rademacher_bound}.
\end{proof}

A linear classifier for this case can be computed by means of a linear optimization program as shown in Section \ref{case:class_linear_compute}. However, prior to discussing this we show the probabilistic error classification guarantees that accompany such a classifier.
\begin{corollary} \label{corollary:linear_with_bounded} \emph{Linear binary classifier with bounded inputs.} Let $\mathcal{H} = \{x \rightarrow a^Tx+b,\ a \in \mathbb{R}^d,\ b \in
  \mathbb{R}\}$. Assume that $\{x \in X| \norm{x^\prime} \leq u\}$ and that $\norm{w}_* \leq v$, where $w = [a^T, b]^T$. We then have that for any $\gamma >1$ and any $r>0$, with probability at least $1-\delta$, for any $\xi \in ]0, \frac{r}{v}]$, and for any linear classifier (parameterized by $a, b$), 
\begin{align*}
  \R{\mathrm{rob}}{h}{\zeta} \leq & \frac{1}{m} \sum_{i=1}^m \phi_{2, v \xi}(y_iw^Tx_i^\prime) + \frac{2 \gamma}{\xi}\sqrt{\frac{u^2}{m}}\\
  & + \sqrt{\frac{\log \log_\gamma \frac{\gamma r}{v \xi}}{m}} + \sqrt{\frac{\log{\frac{2}{\delta}}}{2m}}.
\end{align*}
\end{corollary}
\begin{proof}
The proof follows from Theorem \ref{theorem:main_theorem} and Lemma \ref{lemma:linear_classifier_rademacher_bound} and is provided in Appendix \ref{appendix:proofs:main_theorem:linear_classifier_rademacher_bound:corollary}.
\end{proof}

\subsubsection{Kernel-based non-linear binary classifiers} \label{case:class_kernel}
The so called ``kernel'' approach, commonly used for non-linear classifiers \cite{Book_ML_ShalevBeDavid2014}, embeds the input space into a higher-dimensional feature space and employs a linear classifier there. This enables non-linear classification in the original input space. However, applying this approach can be challenging due to potential infinite-dimensional feature spaces or the need for a large number of sample points to achieve desired accuracy. To address these challenges, kernel-based learning approaches provide a solution.

\begin{definition} \emph{Kernel.} Given an embedding $\psi : X \rightarrow \mathbb{H}$, mapping the domain space into some Hilbert space, we define the Kernel function as $K(x, x^\prime) = \langle \psi(x), \psi(x^\prime) \rangle$, for all $x, x^\prime \in X$.
\end{definition}

\begin{definition} \label{definition:pds_kernel} \emph{Positive Definite Symmetric (PDS) kernels.} A kernel $K: X \times X \rightarrow \mathbb{R}$ is said to be posititive definite symmetric (PDS) if for any ${x_1, \ldots, x_m} \subseteq X$, the matrix $\mathbf{K} = \left[ k(x_i, x_j)\right]_{ij} \in \mathbb{R}^{m \times m}$ is symmetric positive semidefinite. 
\end{definition}

Kernel-based classifiers can be computed by means of a second-order cone program as shown in Section \ref{case:class_kernel_compute}. However, prior to discussing this, the following corollary of Theorem \ref{theorem:main_theorem} shows the probabilistic error classification guarantees that accompany such a classifier.


\begin{corollary} 
\label{corollary:kernel} 
Let $K: X \times X \rightarrow \mathbb{R}$ be a PDS kernel, $\mathbb{H}$ its corresponding RKHS, (Theorem\ref{theorem:rkhs}), equipped a the norm $\norm{\cdot}_{\mathbb{H}}$, and $\psi: X \rightarrow \mathbb{H}$, the feature map associated with it. Let $\mathcal{H} = \{x \rightarrow w^T \psi(x), x \in X, w \in \mathbb{H}\}$. Assume that $\norm{w}_\mathbb{H} \leq v$ and that $K(x, x) < u^2$. We then have that with probability at least $1-\delta$,
\begin{mini*}|s|
    {\alpha^T \mathbf{K} \alpha \leq v^2}
    {\frac{1}{m} \sum_{i=1}^m \phi_{2, \zeta}(y_i (\mathbf{K} \alpha)_i)}
    {}{\Opt{min}{w \in \mathbb{H}}{\R{\mathrm{rob}}{h}{\zeta}}{} -}
\end{mini*}
\begin{equation*}
    \leq \frac{2 \gamma}{\zeta}\sqrt{\frac{u^2v^2}{m}} + \sqrt{\frac{\log \log_\gamma \frac{\gamma r}{\zeta}}{m}} + \sqrt{\frac{\log{\frac{2}{\delta}}}{2m}},
\end{equation*} 
where $\mathbf{K} = [K(x_i, x_j)]_{ij}$ is a symmetric positive semi-definite matrix and $\alpha \in X^m$.
\end{corollary}
\begin{proof}
The proof is provided in Appendix \ref{appendix:proofs:lemma:kernel}.
\end{proof}

\subsubsection{Multi-class classifiers}

For multi-class classifiers with $k$ classes, the preferred method involves employing scoring functions $h$, enabling the classifier to determine the class associated with the highest score. This approach establishes a mapping \cite{Book_ML_Mohri2018} between the input data and the class that yields the maximum score,
\[
  x \in \Opt{argmax}{y \in \{1, \ldots, k\}}{h(x,y)}{}.
\] 
Similar to the binary classification case, it is possible to generalize the concept of confidence margin, by defining
\begin{equation}
  h^\prime(x,y) = h(x,y) - \Opt{max}{y^\prime \neq y}{h(x, y ^\prime)}{}.
\end{equation}
Note that if $h$ misclassifies $(x,y)$, then $h^\prime(x,y) < 0$. In this case, the learner's goal is to minimize the natural classification error,
\begin{equation}
  \R{\mathrm{nat}}{h}{} = \E{\mathbb{P}}{\mathds{1}_{\{x\in X:~(h(x,y) - \Opt{max}{y^\prime \neq y}{h(x, y ^\prime)}{}) \leq 0\}}}.
\end{equation}

We focus on an adversary aiming to deceive the classifier by inducing a mistake, without targeting a specific class-to-class transformation. This involves manipulating the input to be classified as the closest class, without a particular target class. Similar to the binary classifier, we define $\R{\mathrm{rob}}{h}{\zeta}$,
\begin{equation}
  \R{\mathrm{rob}}{h}{\zeta} = \E{\mathbb{P}}{\mathds{1}_{\{x:~ \exists h(\Tilde{x}) \in \mathcal{B}_\zeta(x) | h^\prime(\tilde{x},y) \leq 0\}}}.
\end{equation}
Note the similarity of this definition with proposed in \eqref{equation:R_rob_zeta}.

\begin{theorem}
\label{theorem:main_theorem_multiclass}
\emph{Upper bounded multi-class classifier:} Consider the hypothesis class of scoring functions $\mathcal{H} = \{(x, y) \rightarrow h(x,y)\}$ . We then have that for any $r>0$ and any $\zeta \in ]0, r]$, with probability at least $1-\delta$,
  \begin{align*}
   \R{\mathrm{rob}}{h}{\zeta} \leq & \frac{1}{m} \sum_{i=1}^m \phi_{2, \zeta}(h^\prime(x_i, y_i)) + \frac{2 k \gamma}{\zeta}\Rademacher{m}(\mathcal{H}) \\ 
   & + \sqrt{\frac{\log \log_\gamma \frac{\gamma r}{\zeta}}{m}} + \sqrt{\frac{\log{\frac{2}{\delta}}}{2m}}.
  \end{align*}
\end{theorem}
\begin{proof}
The proof follows the same steps as that of Theorem \ref{theorem:main_theorem}, only deviating on the bounding of the Rademacher complexity $\Rademacher{m}(\mathcal{H}^\prime)$. The proof is provided in Appendix \ref{appendix:proofs:theorem:main_theorem_multiclass}.
\end{proof}

The result is similar to Theorem \ref{theorem:main_theorem}, with only a scaling constant $k$ difference. With more classes, a larger sample size is required for model training, and this dependency is linear with the number of classes.

\subsection{Classifier computation}
\label{subsection:algorithms}
In this section we discuss how to compute linear and kernel-based classifiers using empirical data that enjoy the probabilistic guarantees of Corollaries \ref{corollary:linear_with_bounded} and \ref{corollary:kernel}, respectively.


\subsubsection{Linear programming formulation for linear binary classifiers} \label{case:class_linear_compute}
Let $\norm{x^\prime}_\infty \leq u = r$, $\norm{w}_1 \leq v = 1$, and $r=1$, i.e., we normalize the input to its maximum value, and further consider a powerful adversary, such that $\xi = r = 1$. By Corollary \ref{corollary:linear_with_bounded}, considering the explicit expression of the surrogate function, have that with confidence at least $1 - \delta$, for any linear classifier parameterized by $w \in \mathbb{R}^{d+1}$,
\begin{mini*}|s|
    {\norm{w}_1 \leq 1}
    {\frac{1}{m} \sum_{i=1}^m \big(2 - y_iw^Tx_i^\prime \big)_+}
    {}{\Opt{min}{w}{\R{\mathrm{rob}}{h}{}}{} -}
\end{mini*}
\begin{equation}
       \leq \sqrt{\frac{\log{\frac{2}{\delta}}}{2m}} + 1.
\end{equation} 
where we omitted the superscript $\zeta$ in $\R{\mathrm{rob}}{h}{}$ since $\zeta = 1$ based on the discussion above.
We get this simplified bound by setting $\gamma = \sqrt{m}/2$, while the requirement of $\gamma >1$ holds for any $m>4$. 

A classifier that enjoys such guarantees can be constructed as the solution of the empirical minimization of the second term in the previous equation. This is a minimization subject to a first norm constraint. We could equivalently recast this as a linear program (LP) by introducing some additional decision variables. The resulting optimization program is given by 
\begin{mini}|s|
    {w, t, l, \in \mathbb{R}^{2(d+1)+m}}
    {\frac{1}{m} \sum_{i=1}^m t_i}
    {}{}
    \addConstraint{2 - \frac{y_iw^Tx_i^\prime}{\xi} \le t_i}
    \addConstraint{t_i, l_j \ge 0}
    \addConstraint{w_j \leq l_j,\ -w_j \leq l_j}
    \addConstraint{\sum_{j=1}^{d+1}l_j \leq 1}{}.
\end{mini}
This result is similar to SVM, however, it involves a different norm. We discuss these similarities in more detail in Appendix \ref{appendix:similarities_to_svm}.

\subsubsection{Second order cone programming formulation for kernel-based binary classifiers} \label{case:class_kernel_compute}
By Corollary \ref{corollary:kernel}, we have that with probability at least $1-\delta$,
\begin{mini*}|s|
    {\alpha^T \mathbf{K} \alpha \leq v^2}
    {\frac{1}{m} \sum_{i=1}^m  \Big(2 - \frac{y_i (\mathbf{K} \alpha)_i}{\zeta} \Big)_+}
    {}{\Opt{min}{w \in \mathbb{H}}{\R{\mathrm{rob}}{h}{}}{} -}
\end{mini*}
\begin{equation}
    \leq \sqrt{\frac{\log{\frac{2}{\delta}}}{2m}}  + 1,
\end{equation} 
where $\alpha \in \mathbb{R}^m$. Similarly to the case of linear classifier, we obtained this by letting $r=1$, taking $u = v= \zeta = r = 1$ and setting $\gamma$ as in the previous section. 

The classifier that enjoys such classification guarantees can be obtained as the solution of the empirical minimization problem that appears as the second term in the previous equation. This can be equivalently written as a second-order cone program (SOCP), given by
\begin{mini}|s|
    {\alpha, t_1,\ldots, t_m  \in \mathbb{R}^{2m}}
    {\frac{1}{m} \sum_{i=1}^m t_i}
    {}{}
    \addConstraint{ y_i (\mathbf{K} \alpha)_i \ge 2 - t_i }
    \addConstraint{\ t_i \ge 0,\forall i \in 1, \ldots, m}
    \addConstraint{  \norm{\mathbf{L} \alpha }_2 \leq v^2},
\end{mini}
where $K=\mathbf{L}^T\mathbf{L}$, that is, $\mathbf{L}$ can be obtained through the Cholesky decomposition of $\mathbf{K}$. 
\section{Numerical experiments}

\subsection{Simulation set-up}
In our numerical examples, we conduct a series of experiments using a binary linear classifier applied to the MNIST and CIFAR10 data-sets. We refer to Appendix \ref{appendix:numerical_examples} for the details about the parameters used in the numerical analysis and a reference to the the Github repository with the available code.

\begin{figure*}[h]
  \centering
  \includegraphics[width=\textwidth]{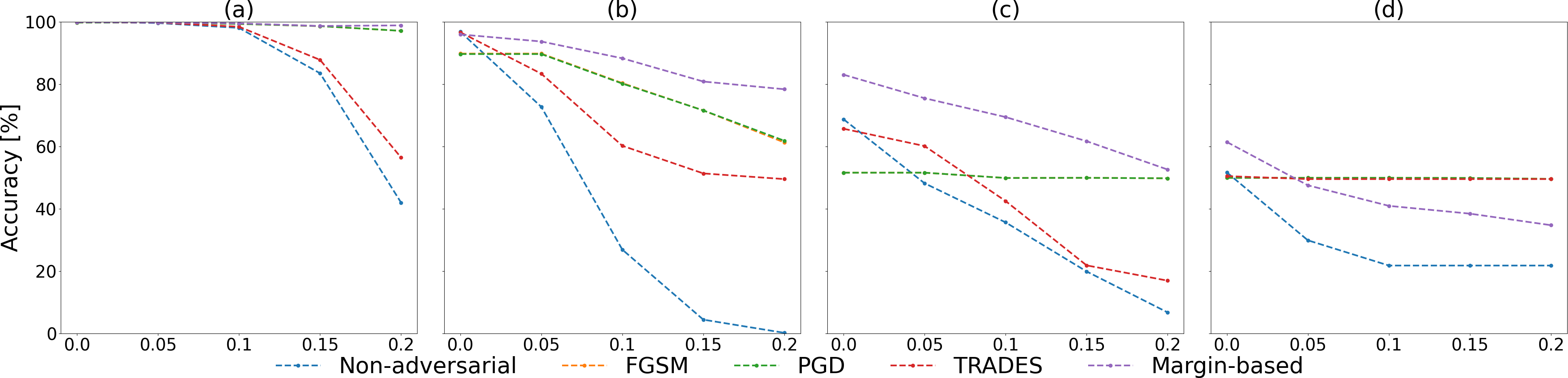}
  \caption{Accuracy of linear classifiers using out-of-sample adversarial tampered data considering non-adversarial training, FGSM \cite{Goodfellow2015}, PGD \cite{Madry2019}, TRADES [$\lambda=1.0$] \cite{TRADES_ZhangElGhaoui2019} and proposed margin-based approach. Datasets: (a) NIST 0/1, (b) NIST 3/8, (c) CIFAR10 Airplane/Dog and (d) CIFAR10 Cat/Dog.} 
  \label{fig:mnist_cifar}
\end{figure*}

We employ the proportion of correctly classified instances in an out-of-sample test set as our accuracy metric for performance evaluation. Our evaluation extends beyond the conventional test case, as all instances in the test set are manipulated with adversarial examples crafted to mislead the model. Despite this challenging scenario, our margin-based model showcases competitive performance, on par with state-of-the-art adversarial training techniques, across both datasets.

\subsection{Simulation results}
When evaluating the NIST dataset, our goal is to differentiate between two digit pairs (0/1 and 3/8) using a binary classifier. Figures \ref{fig:mnist_cifar} (a) and (b) show that standard training is highly vulnerable to adversarial perturbations. For the 0/1 case, both our margin-based approach and the FGSM adversarial training achieve excellent performance, with the latter slightly outperforming under stronger attacks. In the more challenging 3/8 case, our margin-based approach clearly outshines other adversarial training methods.


\begin{table}[h]
    \centering
    \begin{tabular}{cccc}
        \hline
        Data & $\xi$ & Non-adv. [$\%$] & Margin-based [$\%$]\\
        \hline
          \multirow{4}{*}{0/1} &  0.05 & 99.95 (2.17) & 99.91 (3.07) \\
          & 0.10 & 99.95 (2.20) & 99.85 (3.64) \\
          & 0.15 & 99.94 (2.24) & 99.19 (8.33)\\
          & 0.20 & 99.94 (2.27) & 99.52 (5.64)\\
          \hline
          \multirow{4}{*}{3/8} & 0.05 & 96.74 (17.38) & 95.96 (19.29)\\
          & 0.10 & 96.70 (17.15) & 92.05 (26.14)\\
          & 0.15 & 96.67 (16.94) & 85.61 (33.39)\\
          & 0.20 & 96.64 (16.71) & 85.99 (31.22) \\
          \hline
          \multirow{4}{*}{Plane/Dog} &  0.05 & 68.73 (46.20) & 83.01 (36.89)\\
          & 0.10 & 68.73 (46.02) & 79.00 (39.25)\\
          & 0.15 & 68.72 (45.87) & 76.86 (39.42)\\
          & 0.20 & 68.69 (45.72) & 71.58 (41.05)\\
          \hline
          \multirow{4}{*}{Cat/Dog} &  0.05 & 51.70 (49.81) & 61.21 (47.12)\\
          & 0.10 & 51.72 (49.68) & 58.13 (45.77)\\
          & 0.15 & 51.74 (49.58) & 57.07 (42.27)\\
          & 0.20 & 51.74 (49.49) & 56.21 (37.76)\\
          \hline
    \end{tabular}
    \label{table:roma}
    \caption{RoMa score:  mean and standard deviation for non-adversarial and margin-based training methods.}
\end{table}

In the context of the CIFAR10 dataset, it is crucial to note that a linear classifier exhibits low accuracy even without any adversarial influence. Specifically, in the cat/dog classification, a linear model struggles even without adversaries, performing no better than chance, as depicted in Figure \ref{fig:mnist_cifar} (d). However, our margin-based approach demonstrates robust classification, achieving notable accuracy even in challenging scenarios such as distinguishing between airplane/dog, outperforming all methods, and surpassing usual training in the cat/dog case, as shown in  Figure \ref{fig:mnist_cifar} charts (c) and (d).

Robustness against adversarial inputs is a crucial factor, evaluated using the RoMA (Robustness Measurement and Assessment) procedure \cite{RoMA_LevyKatz2022}. This method determines the probability of a random input perturbation causing a misclassification, providing guarantees on the expected error frequency post-training \cite{RoMA_LevyKatz2022}. In this metric (Table \ref{table:roma}), both the proposed margin-based and conventional training methods show comparable robustness, especially in scenarios where adversaries find it challenging to execute attacks, as observed in the NIST 0/1 dataset. However, the margin-based approach demonstrates significantly higher robustness on datasets with lower accuracy under conventional training, such as the CIFAR10.

\section{Conclusion}

We focused on robust classification under adversarial attacks and introduced a new method for adversarial training, inspired by SVM margin concepts. We established finite sample complexity bounds that accompany adversarialy trained classifiers with probabilistic error classification guarantees. Moreover, we showed that robust linear and kernel-based binary classifiers can be constructed by means of a linear and a second-order cone program respectively. Extensive numerical validation was provided.
A distinctive feature of the proposed methodology is the ability to achieve high accuracy without generating adversarial examples during training. 


\IEEEpeerreviewmaketitle

\bibliographystyle{ieeetr}
\bibliography{main.bib}

\begin{thebibliography}{10}

\bibitem{NileshDomingos2004}
N.~Dalvi, P.~Domingos, Mausam, S.~Sanghai, and D.~Verma, ``Adversarial
  classification,'' in {\em Proceedings of the Tenth ACM SIGKDD International
  Conference on Knowledge Discovery and Data Mining}, KDD '04, (New York, NY,
  USA), p.~99–108, Association for Computing Machinery, 2004.

\bibitem{Szegedy2014}
C.~Szegedy, W.~Zaremba, I.~Sutskever, J.~Bruna, D.~Erhan, I.~Goodfellow, and
  R.~Fergus, ``Intriguing properties of neural networks,'' 2014.

\bibitem{Goodfellow2015}
I.~J. Goodfellow, J.~Shlens, and C.~Szegedy, ``Explaining and harnessing
  adversarial examples,'' 2015.

\bibitem{Battista2018}
B.~Biggio and F.~Roli, ``Wild patterns: Ten years after the rise of adversarial
  machine learning,'' {\em Pattern Recognition}, vol.~84, pp.~317 -- 331, 2018.

\bibitem{Anirban2018}
A.~Chakraborty, M.~Alam, V.~Dey, A.~Chattopadhyay, and D.~Mukhopadhyay,
  ``Adversarial attacks and defences: {A} survey,'' {\em CoRR},
  vol.~abs/1810.00069, 2018.

\bibitem{Ram2019}
R.~S.~S. Kumar, D.~R. O'Brien, K.~Albert, S.~Vilj{\"{o}}en, and J.~Snover,
  ``Failure modes in machine learning systems,'' {\em CoRR},
  vol.~abs/1911.11034, 2019.

\bibitem{Yuan2019}
X.~Yuan, P.~He, Q.~Zhu, and X.~Li, ``Adversarial examples: Attacks and defenses
  for deep learning,'' {\em IEEE Transactions on Neural Networks and Learning
  Systems}, vol.~30, no.~9, pp.~2805--2824, 2019.

\bibitem{BaiTao2021}
T.~Bai, J.~Luo, J.~Zhao, B.~Wen, and Q.~Wang, ``Recent advances in adversarial
  training for adversarial robustness,'' in {\em Proceedings of the Thirtieth
  International Joint Conference on Artificial Intelligence, {IJCAI-21}} (Z.-H.
  Zhou, ed.), pp.~4312--4321, International Joint Conferences on Artificial
  Intelligence Organization, 8 2021.
\newblock Survey Track.

\bibitem{Kukain2016}
A.~Kurakin, I.~J. Goodfellow, and S.~Bengio, ``Adversarial examples in the
  physical world,'' {\em CoRR}, vol.~abs/1607.02533, 2016.

\bibitem{Madry2019}
A.~Madry, A.~Makelov, L.~Schmidt, D.~Tsipras, and A.~Vladu, ``Towards deep
  learning models resistant to adversarial attacks,'' 2019.

\bibitem{Papernot2016}
N.~{Papernot}, P.~{McDaniel}, X.~{Wu}, S.~{Jha}, and A.~{Swami}, ``Distillation
  as a defense to adversarial perturbations against deep neural networks,'' in
  {\em 2016 IEEE Symposium on Security and Privacy (SP)}, pp.~582--597, 2016.

\bibitem{Robey2022}
A.~Robey, L.~Chamon, G.~J. Pappas, and H.~Hassani, ``Probabilistically robust
  learning: Balancing average and worst-case performance,'' in {\em Proceedings
  of the 39th International Conference on Machine Learning} (K.~Chaudhuri,
  S.~Jegelka, L.~Song, C.~Szepesvari, G.~Niu, and S.~Sabato, eds.), vol.~162 of
  {\em Proceedings of Machine Learning Research}, pp.~18667--18686, PMLR,
  17--23 Jul 2022.

\bibitem{Book_CampiGaratti2018}
M.~C. Campi and S.~Garatti, {\em Introduction to the Scenario Approach}.
\newblock Philadelphia, PA: Society for Industrial and Applied Mathematics,
  2018.

\bibitem{CampiGaratti2008}
M.~C. Campi and S.~Garatti, ``The exact feasibility of randomized solutions of
  uncertain convex programs,'' {\em SIAM Journal on Optimization}, vol.~19,
  no.~3, pp.~1211--1230, 2008.

\bibitem{GarattiCampi2019}
S.~Garatti and M.~C. Campi, ``Risk and complexity in scenario optimization,''
  {\em Mathematical Programming}, vol.~191, pp.~243--279, Jan 2022[2019].

\bibitem{CampiGaratti2021}
M.~C. Campi and S.~Garatti, ``A theory of the risk for optimization with
  relaxation and its application to support vector machines,'' {\em Journal of
  Machine Learning Research}, vol.~22, no.~288, pp.~1--38, 2021.

\bibitem{Margellos2014}
K.~Margellos, M.~Prandini, and J.~Lygeros, ``A compression learning perspective
  to scenario based optimization,'' in {\em 53rd IEEE Conference on Decision
  and Control}, pp.~5997--6002, 2014.

\bibitem{Campi2023}
M.~C. Campi and S.~Garatti, ``Compression, generalization and learning,'' {\em
  Journal of Machine Learning Research}, vol.~24, no.~339, pp.~1--74, 2023.

\bibitem{Pauli2022}
P.~Pauli, A.~Koch, J.~Berberich, P.~Kohler, and F.~Allgöwer, ``Training robust
  neural networks using lipschitz bounds,'' {\em IEEE Control Systems Letters},
  vol.~6, pp.~121--126, 2022.

\bibitem{Mahyar2019}
M.~Fazlyab, A.~Robey, H.~Hassani, M.~Morari, and G.~Pappas, ``Efficient and
  accurate estimation of lipschitz constants for deep neural networks,'' in
  {\em Advances in Neural Information Processing Systems} (H.~Wallach,
  H.~Larochelle, A.~Beygelzimer, F.~d\textquotesingle Alch\'{e}-Buc, E.~Fox,
  and R.~Garnett, eds.), vol.~32, Curran Associates, Inc., 2019.

\bibitem{CortesVapnik1995}
C.~Cortes and V.~Vapnik, ``Support vector networks,'' {\em Machine Learning},
  vol.~20, pp.~273--297, 1995.

\bibitem{Attias2019}
I.~Attias, A.~Kontorovich, and Y.~Mansour, ``Improved generalization bounds for
  robust learning,'' in {\em Proceedings of the 30th International Conference
  on Algorithmic Learning Theory} (A.~Garivier and S.~Kale, eds.), vol.~98 of
  {\em Proceedings of Machine Learning Research}, pp.~162--183, PMLR, 22--24
  Mar 2019.

\bibitem{DongBarlett2019}
D.~Yin, R.~Kannan, and P.~Bartlett, ``Rademacher complexity for adversarially
  robust generalization,'' in {\em Proceedings of the 36th International
  Conference on Machine Learning} (K.~Chaudhuri and R.~Salakhutdinov, eds.),
  vol.~97 of {\em Proceedings of Machine Learning Research}, pp.~7085--7094,
  PMLR, 09--15 Jun 2019.

\bibitem{Bhattacharjee2021}
R.~Bhattacharjee, S.~Jha, and K.~Chaudhuri, ``Sample complexity of robust
  linear classification on separated data,'' in {\em Proceedings of the 38th
  International Conference on Machine Learning} (M.~Meila and T.~Zhang, eds.),
  vol.~139 of {\em Proceedings of Machine Learning Research}, pp.~884--893,
  PMLR, 18--24 Jul 2021.

\bibitem{Ashtiani2023}
H.~Ashtiani, V.~Pathak, and R.~Urner, ``Adversarially robust learning with
  tolerance,'' in {\em Proceedings of The 34th International Conference on
  Algorithmic Learning Theory} (S.~Agrawal and F.~Orabona, eds.), vol.~201 of
  {\em Proceedings of Machine Learning Research}, pp.~115--135, PMLR, 20
  Feb--23 Feb 2023.

\bibitem{TRADES_ZhangElGhaoui2019}
H.~Zhang, Y.~Yu, J.~Jiao, E.~Xing, L.~E. Ghaoui, and M.~Jordan, ``Theoretically
  principled trade-off between robustness and accuracy,'' in {\em Proceedings
  of the 36th International Conference on Machine Learning} (K.~Chaudhuri and
  R.~Salakhutdinov, eds.), vol.~97 of {\em Proceedings of Machine Learning
  Research}, pp.~7472--7482, PMLR, 09--15 Jun 2019.

\bibitem{SMART_WangZou2020}
Y.~Wang, D.~Zou, J.~Yi, J.~Bailey, X.~Ma, and Q.~Gu, ``Improving adversarial
  robustness requires revisiting misclassified examples,'' in {\em ICLR}, 2020.

\bibitem{Cullina2018}
D.~Cullina, A.~N. Bhagoji, and P.~Mittal, ``Pac-learning in the presence of
  evasion adversaries,'' in {\em Proceedings of the 32nd International
  Conference on Neural Information Processing Systems}, NIPS'18, (Red Hook, NY,
  USA), p.~228–239, Curran Associates Inc., 2018.

\bibitem{FeigeMansour2015}
U.~Feige, Y.~Mansour, and R.~Schapire, ``Learning and inference in the presence
  of corrupted inputs,'' in {\em Proceedings of The 28th Conference on Learning
  Theory} (P.~Grünwald, E.~Hazan, and S.~Kale, eds.), vol.~40 of {\em
  Proceedings of Machine Learning Research}, (Paris, France), pp.~637--657,
  PMLR, 03--06 Jul 2015.

\bibitem{LeeRaginsky2018}
J.~Lee and M.~Raginsky, ``Minimax statistical learning with wasserstein
  distances,'' in {\em Advances in Neural Information Processing Systems}
  (S.~Bengio, H.~Wallach, H.~Larochelle, K.~Grauman, N.~Cesa-Bianchi, and
  R.~Garnett, eds.), vol.~31, Curran Associates, Inc., 2018.

\bibitem{TuZhang2019}
Z.~Tu, J.~Zhang, and D.~Tao, ``Theoretical analysis of adversarial learning: A
  minimax approach,'' in {\em Advances in Neural Information Processing
  Systems}, vol.~32, Curran Associates, Inc., 2019.

\bibitem{Book_ML_ShalevBeDavid2014}
S.~Shalev-Shwartz and S.~Ben-David, {\em Understanding Machine Learning: From
  Theory to Algorithms}.
\newblock USA: Cambridge University Press, 2014.

\bibitem{Book_ML_Mohri2018}
M.~Mohri, A.~Rostamizadeh, and A.~Talwalkar, {\em Foundations of Machine
  Learning}.
\newblock The MIT Press, 2nd~ed., 2018.

\bibitem{RoMA_LevyKatz2022}
N.~Levy and G.~Katz, ``Roma: A method for neural network robustness measurement
  and assessment,'' in {\em Proc. 29th Int. Conf. on Neural Information
  Processing (ICONIP)}, pp.~92--105, November 2022.

\bibitem{Carlini2017}
N.~{Carlini} and D.~{Wagner}, ``Towards evaluating the robustness of neural
  networks,'' in {\em 2017 IEEE Symposium on Security and Privacy (SP)},
  pp.~39--57, 2017.

\bibitem{DeepFool2016}
S.-M. Moosavi-Dezfooli, A.~Fawzi, and P.~Frossard, ``Deepfool: A simple and
  accurate method to fool deep neural networks,'' in {\em 2016 IEEE Conference
  on Computer Vision and Pattern Recognition (CVPR)}, pp.~2574--2582, 2016.

\bibitem{Biggio2011}
B.~Biggio, B.~Nelson, and P.~Laskov, ``Support vector machines under
  adversarial label noise,'' in {\em Proceedings of the Asian Conference on
  Machine Learning} (C.-N. Hsu and W.~S. Lee, eds.), vol.~20 of {\em
  Proceedings of Machine Learning Research}, (South Garden Hotels and Resorts,
  Taoyuan, Taiwain), pp.~97--112, PMLR, 14--15 Nov 2011.

\end{thebibliography}

\appendices
\section{Definitions and Theorems}
\label{appendix:definitions_and_theorems}

\begin{definition}
\label{definition:Rademacher}
  \emph{Empirical Rademacher complexity}
  Let $\mathcal{H}$ be a family of functions mapping from $Z$ to $[a,b]$ and $S = (s_1, s_2, \ldots, s_m)$ a fixed sample of size $m$ with elements in $Z$. Then, the empirical Rademacher complexity of $\mathcal{H}$ with respect to the sample $S$ is defined as,

  \begin{equation*}
    \EmpRademacher{S}(\mathcal{H}) = \E{\mathbb{B}}{
      \Opt{sup}
      {h \in \mathcal{H}}
      {\frac{1}{m} \sum_{i=1}^m \sigma_i h(s_i)}
      {}},
  \end{equation*}
  where $\sigma_i, i=1, \ldots, m$ are iid random variables following a symmetric Bernoulli (also called Rademacher) distribution, that is, $\sigma$ takes values in $\{-1, 1\}$ with probability $\frac{1}{2}$,
  \begin{equation*}
    \mathbb{B}\{\sigma = 1 \} = \mathbb{B}\{\sigma = -1 \} = \frac{1}{2}.
  \end{equation*}

\end{definition} 

\begin{definition}
  \emph{Rademacher complexity} Let $\mathbb{P}$ denote the distribution according to which samples are drawn. For any integer $m \geq 1$, the Rademacher complexity of $\mathcal{H}$ is the expectation of the empirical Rademacher complexity over all samples of size $m$ according to $\mathbb{P}$,
  \begin{equation*}
    \Rademacher{m}(\mathcal{H}) = \E{\mathbb{P}^m}{\EmpRademacher{S}(\mathcal{H})} .
  \end{equation*}

\end{definition} 


\begin{definition}
\label{def:PAC}
  \emph{Agnostic PAC-learning} \cite{Book_ML_Mohri2018} A hypothesis class $\mathcal{H}$ is   called agnostic PAC-learnable if there exists an algorithm $M$, that returns a hypothesis $h_S$ given the training sample $S$ and a polynomial function $p(.,.,. ,. )$ such that for any $\epsilon, \delta> 0$, \emph{for all distributions} $\mathcal{P}$ over $Z = X \times Y$, the following holds for any sample size $m \geq p(1/\epsilon, 1/\delta,dim(X))$ :
  \[
    \mathcal{P}^m\Big\{\E{\mathcal{P}}{\ell (h_S(\chi), \upsilon)} - \Opt{min}{h \in
      \mathcal{H}}{\E{\mathcal{P}}{\ell (h(\chi), \upsilon)}}{} \leq \epsilon \Big\} \geq 1-\delta.
  \]
\end{definition}

\begin{theorem} (Theorem 3.3 at \cite{Book_ML_Mohri2018}) 
\label{theorem:rademacher_bound}
  Let $\mathcal{G}$ be a family of functions mapping from $X \rightarrow [0,1]$. Then, for any $\delta>0$, with probability at least $1-\delta$ over the draw of any iid sample $S$ of size $m$, each of the following holds for all $g \in \mathcal{G}$,
  \begin{align*}
    \E{}{g(x)} &\leq \Eemp{S}{g(s_i)} + 2\Rademacher{m}(\mathcal{G}) + \sqrt{\frac{\log{\frac{1}{\delta}}}{2m}} \\
    \E{}{g(x)} &\leq \Eemp{S}{g(s_i)} + 2\EmpRademacher{S}(\mathcal{G}) + 3\sqrt{\frac{\log{\frac{2}{\delta}}}{2m}},
  \end{align*}
  with $\Eemp{S}{g(s_i)} = \frac{1}{m} \sum_{i=1}^m g(s_i)$.
\end{theorem}

\begin{theorem} (Theorem 3.5 at \cite{Book_ML_Mohri2018}) 
\label{theorem:rademacher_bound_binary_classifier}
  Let $\mathcal{H}$ be a family of functions taking values in $\{-1,1\}$ and let $\mathcal{D}$ be the distribution over the input space $X$. Then, for any $\delta > 0$, with probability at least $1-\delta$ over a sample $S$ of size $m$ drawn according to $\mathcal{D}$, each of the following holds for any $h \in \mathcal{H}$:
  \begin{align*}
  R(h) &\leq \hat{R}(h) + \Rademacher{m}(\mathcal{H}) + \sqrt{\frac{\log{\frac{1}{\delta}}}{2m}} \\
  R(h) &\leq \hat{R}(h) + \EmpRademacher{S}(\mathcal{H}) + 3\sqrt{\frac{\log{\frac{1}{\delta}}}{2m}}.
  \end{align*}
\end{theorem}

\begin{theorem} (Theorem 9.3 at \cite{Book_ML_ShalevBeDavid2014})
  The VC dimension of the class of non-homogeneous half-spaces in $\mathbb{R}^N$ is $N+1$
\end{theorem}

\begin{theorem} (Corollary 3.19 at \cite{Book_ML_Mohri2018})
\label{Theorem:VCdim_binary}
  Let $\mathcal{H}$ be a family of functions taking values in $\{ -1, +1\}$ with VC-dimension $d$. Then, for any $\delta > 0$, with probability at least $1-\delta$, the following holds for all $h \in \mathcal{H}$:
    \begin{equation*}
    \R{}{h}{} \leq \Remp{S}{h} + \sqrt{\frac{2d\log{\frac{em}{d}}}{m}} + \sqrt{\frac{\log{\frac{1}{\delta}}}{2m}}
  \end{equation*}
\end{theorem}

\begin{lemma} \emph{Talagrand's lemma} (Lemma 5.7 at \cite{Book_ML_Mohri2018}) Let $\Phi_1, \ldots, \Phi_m$
\label{lemma:talagrand}
  be $l$-Lipschitz functions from $\mathbb{R} \rightarrow \mathbb{R}$ and $\sigma_1, \ldots, \sigma_m$ be Rademacher random variables. Then, for any hypothesis set $\mathcal{H}$ of real-valued functions, the following inequality holds,
    \begin{align*}
    \frac{1}{m} & \E{\mathbb{Q}^m}{
      \Opt{sup}
      {h \in \mathcal{H}}
      {\sum_{i=1}^m \sigma_i (\Phi_i \circ h)(z_i)}
      {}} \leq \\
    & \frac{l}{m} \E{\mathbb{Q}^m}{
      \Opt{sup}
      {h \in \mathcal{H}}
      {\sum_{i=1}^m \sigma_i h(z_i)}
      {}} = l \EmpRademacher{S}(\mathcal{H}).
  \end{align*}
  In particular, if $\Phi_i = \Phi$ for all $i \in {1, \ldots, m}$, then the following holds,
  \begin{equation*}
    \EmpRademacher{S}(\Phi \circ \mathcal{H}) \leq l \EmpRademacher{S}(\mathcal{H}).
  \end{equation*}
\end{lemma}

\begin{theorem} \label{theorem:rkhs} \emph{Reproducing kernel Hilbert space (RKHS)} Let $K: X \times X \rightarrow \mathbb{R}$ be a PDS kernel. Then, there exists a Hilbert space $\mathbb{H}$ of functions $f$ and a mapping $\psi: X \rightarrow \mathbb{H}$ such that: 
\[
\forall x,x^\prime \in X,\  K(x, x^\prime) = \langle \psi(x), \psi(x^\prime) \rangle.
\] 
Furthermore, $\mathbb{H}$ has the following property known as the reproducing property:
\[
\forall f \in \mathbb{H},\ \forall x \in X,\  f(x) = \langle f, K(x, \cdot) \rangle.
\] 
$\mathbb{H}$ is called the reproducing kernel Hilbert space (RKHS) associated to K.
\end{theorem}

\begin{theorem} \label{theorem:representer} \emph{Representer theorem} Let $K: X \times X \rightarrow \mathbb{R}$ be a PDS kernel and $\mathbb{H}$ its corresponding RKHS. Then for any non-decreasing function $G: \mathbb{R} \rightarrow \mathbb{R}$ and any loss function $L: \mathbb{R}^m \rightarrow \mathbb{R} \cup {+ \infty}$, the optimization problem
\[
\Opt{argmin}{f \in \mathbb{H}}{F(h)}{} = \Opt{argmin}{f \in \mathbb{H}}{G(\norm{f}_\mathbb{H}) + L(h(x_1), \ldots, h(x_m))}{} 
\]
admits a solution of the form $f^*=\sum_{i=1}^m\alpha_i K(x_i, \cdot)$. If $G$ is further assumed to be increasing, then any solution has this form.
\end{theorem}

\section{Adversarial attacks}
\label{appendix:attacks}

\subsection{Fast gradient sign method (FGSM)}
\label{appendix:attacks:fgsm}
FGSM is an attack for an $\ell_\infty$-bounded adversary \cite{Goodfellow2015} and computes an adversarial example as,
\[
x + \xi \cdot \text{sign}(\nabla_x \ell(h_\theta(x, y)).
\]

One interpretation is that this attack is a simple one-step scheme for maximizing the inner part of the adversarial problem.

\subsection{Projected gradient descent (PGD)}
\label{appendix:attacks:pgd}
A more powerfull attack \cite{Madry2019} is a multi-step variant of the FGSM,
\[
x_{t+1} = \Pi_{\mathcal{B}_\xi(x)} (x_t + \eta \cdot \text{sign}(\nabla_x \ell(h_\theta(x, y))).
\]

\subsection{Carlini and Wagner (CW)}
\label{appendix:attacks:cw}
This attack was proposed in \cite{Carlini2017} as a response to one approach to defend against adversarial attacks,
\begin{mini*}|s|
    {\xi}
    {\norm{\xi}_p + c f(x + \xi)}
    {}{}
    \addConstraint{x + \xi \in [0,1]^n}.
\end{mini*}

The $\norm{.}_p$ measures the distance of the adversarial perturbation and the function $f$ denotes a customized adversarial loss satisfying $f(x + \xi) \leq 0$. The constraint $x + \xi \in [0,1]^n$ ensures that the image generated is a valid one.

\subsection{DeepFool (DF)}
\label{appendix:attacks:df}
DeepFool is based on an iterative linearization of the classifier to generate minimal perturbations that are sufficient to change classification labels \cite{DeepFool2016}.  Specifically, at each iteration, $h$ is linearized around the current point $x$ and the minimal perturbation of the linearized classifier is computed as,
\begin{mini*}|s|
    {\xi}
    {\norm{\xi}_2}
    {}{}
    \addConstraint{h(x) + \nabla_f(x)^T \xi = 0}.
\end{mini*}
\section{Adversarial training examples}
\label{appendix:examples}

\subsection{Linear regression}
\label{appendix:linear_regression}

Consider now the hypothesis class of linear functionals, $\mathcal{H} = \{h_{a, b}: x \rightarrow a^T x + b \ \big| \ a \in \mathbb{R}^d, b \in \mathbb{R} \}$. The learner wishes to find $a,b$ that solves the following:

\begin{align*}
  &\Opt{min}{a,b \in R^{d+1}}{\frac{1}{2}{\E{\mathbb{P}}{\norm{y - (a^Tx + b)}_2^2}}}{}
\end{align*}

To train against adversarial attacks the learner considers the robust counterpart of the above problem. This is referred to adversarial risk minimization in the literature:
\begin{maxi*}|s|
    {\delta}
     {\frac{1}{2}\norm{y - (a^T(x + \delta) + b)}_2^2}
    {}{\Opt{min}{a, b \in R^{d+1}}{}{}}
    \addConstraint{\norm{\delta} \leq \xi}.
\end{maxi*}

\subsubsection{ERM}
Given samples $ S = \big((x_1, y_1), \ldots, (x_m, y_m)\big)$ the learner proceeds by finding $a, b$ that minimizes the quadratic empirical loss:

\begin{align*}
  &\Opt{min}{a,b \in R^d}{\frac{1}{2m}\sum_{i=1}^m{\norm{y_i - (a^Tx_i + b)}_2^2}}{}
\end{align*}

One approach to solving this optimization problem is to rely on gradient descent methods. Yet, an alternative is to use OLS procedure. For this, let us fist write the problem in matrix form:
\begin{equation*}
  \Opt{min}{\theta \in R^d}{\frac{1}{2m}{\norm{Y - X \theta}_2^2}}{}
\end{equation*}

with $X = [[x_1^T, 1], \ldots, [x_m^T, 1]]^T$, 
$Y = [y_1, \ldots, y_m]^T$ and
and  $\theta = [a^T, b]^T$.

Through OLS the optimizal solution is:
\begin{equation*}
  \theta^*  =
  \left\{
    \begin{aligned}
      (X^TX)^{-1}X^TY \  &\text{, if } X^TX \text{ is invertible}\\
      (X^TX)^\dagger X^TY \  &\text{, otherwise}
  \end{aligned}
  \right.
\end{equation*}

\subsubsection{AERM}

\begin{maxi*}|s|
    {\delta \in \mathbb{R}^d}
     {\norm{y_i - (a^T(x_i+ \delta) + b)}_2^2}
    {}{\Opt{min}{a, b \in R^{d+1}}{\frac{1}{m}\sum_{i=1}^m}{}}
    \addConstraint{\norm{\delta} \leq \xi}.
\end{maxi*}

The optimal $\delta^*$ that solves the inner maximization problem is the solution of either,
\begin{maxi*}|s|
    {\delta \in \mathbb{R}^d}
     {\norm{y_i - (a^T(x_i+ \delta) + b)}}
    {}{}
    \addConstraint{\norm{\delta} \leq \xi},
\end{maxi*}
or,
\begin{mini*}|s|
    {\delta \in \mathbb{R}^d}
     {\norm{y_i - (a^T(x_i+ \delta) + b)}}
    {}{}
    \addConstraint{\norm{\delta} \leq \xi},
\end{mini*}

and solving this is equivalent to solving the following problems:
\begin{maxi*}|s|
    {\delta \in \mathbb{R}^d}
     {-a^T \delta}
    {}{y_i - (a^Tx_i + b) +}
    \addConstraint{\norm{\delta} \leq \xi},
\end{maxi*}
\begin{mini*}|s|
    {\delta \in \mathbb{R}^d}
     {-a^T \delta}
    {}{y_i - (a^Tx_i + b) +}
    \addConstraint{\norm{\delta} \leq \xi}.
\end{mini*}

By the dual norm definition, $\norm{y}_* = \Opt{sup}{x}{y^Tx}{\norm{x} \leq 1}$, the above can be rewritten as:
\begin{equation*}
  y_i - (a^Tx_i + b) +\left\{
    \begin{aligned}
      \xi \norm{a}_* \\
      -\xi \norm{a}_*
    \end{aligned}
    \right.
  \end{equation*}

or, in a simple form:
  
\begin{equation*}
  y_i - (a^Tx_i + b) -\xi \norm{a}_*
\end{equation*}

As a result, the AERM becomes:

\begin{equation*}
  \Opt{min}
  {a, b \in \mathbb{R}^d}
  {\frac{1}{m}\sum_{i=1}^m
  \norm{y_i - (a^Tx_i + b + \xi \norm{a}_*)}_2^2}
  {}
\end{equation*}

\subsection{Example: logistic binary classifier}
\label{appendix:logistic_regression}

In this case $y \in \{0,1\}$. We will be working wiht the class of logistic regression models $\mathcal{H} = \{h_{a, b}: x \rightarrow 1/(1 + e^{-(a^T x + b)}) \
\big| \ a \in \mathbb{R}^d, b \in \mathbb{R} \}$.

\subsubsection{ERM}

In this case the learner wishes to find $a,b$ that minimizes the cross entropy,
\begin{mini*}|s|
    {a,b \in R^{d+1}}
     {\mathbb{E}_\mathbb{P} \big[ y f_{a, b}(x) + (1-y) g_{a,b}(x) \big]}
    {}{},
\end{mini*}
where,
\begin{equation*}
    f_{a,b}(x) = \log( 1 + e^{-(a^T x + b)} ),
\end{equation*}
\begin{equation*}
    g_{a,b}(x) = \log( 1 + e^{(a^T x + b)} ).
\end{equation*}.

The robust counterpart is,
\begin{maxi*}|s|
    {\delta \in \mathbb{R}^{d}}
     {\mathbb{E}_\mathbb{P} \big[y f_{a, b}(x + \delta) + (1-y) g_{a,b}(x + \delta) \big]}
    {}{\Opt{min}{a,b \in \mathbb{R}^{d+1}}{}{}}
    \addConstraint{\norm{\delta} \leq \xi}.
\end{maxi*}

\subsubsection{AERM}

\begin{maxi*}|s|
    {\delta \in \mathbb{R}^{d}}
     {y_i f_{a, b}(x_i + \delta) + (1-y_i) g_{a,b}(x_i + \delta)}
    {}{\Opt{min}{a,b \in \mathbb{R}^{d+1}}{\frac{1}{m}\sum_{i=1}^m}{}}
    \addConstraint{\norm{\delta} \leq \xi}.
\end{maxi*}

Let us focus first on the inner maximization problem,
\begin{maxi*}|s|
    {\delta \in \mathbb{R}^{d}}
     {y_i f_{a, b}(x_i + \delta) + (1-y_i) g_{a,b}(x_i + \delta)}
    {}{}
    \addConstraint{\norm{\delta} \leq \xi}.
\end{maxi*}

If $y_i = 1$, the maximization can be simplified to,
\begin{maxi*}|s|
    {\delta \in \mathbb{R}^d}
     {\log( 1 + e^{-(a^T (x_i + \delta) + b)} )}
    {}{}
    \addConstraint{\norm{\delta} \leq \xi}.
\end{maxi*}

The objective function in this case is monotonically decreasing, finding $\delta$ that solves the maximization is equivalent to finding $\delta$ that solves,
\begin{mini*}|s|
    {\delta \in \mathbb{R}^d}
     {(a^T(x_i+ \delta) + b)}
    {}{}
    \addConstraint{\norm{\delta} \leq \xi},
\end{mini*}
\begin{mini*}|s|
    {\delta \in \mathbb{R}^d}
     {a^T \delta}
    {}{a^Tx_i + b + }
    \addConstraint{\norm{\delta} \leq \xi}.
\end{mini*}

Note that, $\norm{y}_* = \Opt{sup}{x}{y^Tx}{\norm{x} \leq 1}$, by the definition of the dual norm.

\begin{mini*}|s|
    {\delta \in \mathbb{R}^d}
     {a^T \delta}
    {}{ -\xi \norm{a}_* = }
    \addConstraint{\norm{\delta} \leq \xi}.
\end{mini*}

The maximization becomes,
\begin{maxi*}|s|
    {\delta \in \mathbb{R}^d}
    {\log( 1 + e^{-(a^T (x_i + \delta) + b)} )}
    {}{}
    \addConstraint{\norm{\delta} \leq \xi}
\end{maxi*}
\begin{equation*}
    = \log(1 + e^{-(a^T x_i + b + \xi \norm{a}_*)}).
\end{equation*}

Now, if $y_i = 0$,
\begin{maxi*}|s|
    {\delta \in \mathbb{R}^d}
    {\log( 1 + e^{-(a^T (x_i + \delta) + b)} )}
    {}{}
    \addConstraint{\norm{\delta} \leq \xi}.
\end{maxi*}

In this case, the objective function is monotonically increasing, and proceeding on a similar manner, leads to solving,
\begin{maxi*}|s|
    {\delta \in \mathbb{R}^d}
    {(a^T(x_i+ \delta) + b)}
    {}{}
    \addConstraint{\norm{\delta} \leq \xi},
\end{maxi*}
\begin{maxi*}|s|
    {\delta \in \mathbb{R}^d}
    {a^T \delta}
    {}{a^Tx_i + b + }
    \addConstraint{\norm{\delta} \leq \xi}.
\end{maxi*}

And The maximization becomes,
\begin{maxi*}|s|
    {\delta \in \mathbb{R}^d}
    {\log( 1 + e^{(a^T (x_i + \delta) + b)} )}
    {}{}
    \addConstraint{\norm{\delta} \leq \xi}
\end{maxi*}
\begin{equation*}
    = \log(1 + e^{(a^T x_i + b + \xi \norm{a}_*)}).
\end{equation*}

Finally, the simplified form of the AERM is,
\begin{mini*}|s|
    {a,b \in \mathbb{R}^{d+1}}
    {\frac{1}{m}\sum_{i=1}^m y_i \log( 1 + e^{-(a^T x_i + b + \xi \norm{a}_*)} ) + \\
    (1-y_i) \log( 1 + e^{a^T x_i + b + \xi \norm{a}_*} )}
    {}{}
\end{mini*}

\section{SVM and margin theory}
\label{appendix:svm_and_margin_theory}

The SVM algorithm \cite{CortesVapnik1995} had a profound impact on machine learning theory and applications. It was firstly introduced to solve a binary classification problem. It aims at finding the linear hyperplane that maximizes the distance between the closest training samples of the two classes, reducing the generalization error. In addition to performing linear classification, SVMs can efficiently perform non-linear classification by using the so called "kernel approach", which involves mapping the inputs into a higher dimension space.

Initially designed for separable data (hard-margin SVM), it was later extended to handle non-separable classification problems (soft-margin SVM). The goal of the SVM algorithm is to find the hyperplane (parameterized through $w, b$) that maximizes the geometric margin (equivalent to minimizing $\norm{w}_2$), determined by the Euclidean distance from any point to the hyperplane. The hard-margin case is equivalent \cite{Book_ML_Mohri2018} to, 
\begin{mini*}|s|
    {w, b}
    {\frac{1}{2} \norm{w}_2^2}
    {}{}
    \addConstraint{y_i(w^Tx_i + b) \geq 1,\ \forall i \in 1 \ldots m}
\end{mini*}
For the non-separable case, one needs to introduce the slack variables $t_i$ and determine a trade-off between the margin maximization and the minimization of the slack variables penalty. In this case, the SVM takes the following form,
\begin{mini*}|s|
    {w, b, t_1, \ldots, t_m}
    {\frac{1}{2} \norm{w}_2^2 + \lambda \sum_{i=1}^{m} t_i^p}
    {}{}
    \addConstraint{y_i(w^Tx_i + b) \geq 1 - t_i}
    \addConstraint{t_i \geq 0,\ \forall i \in 1 \ldots m}{}.
\end{mini*}

When considering a binary classification problem, we have two approaches for the classifier $h^\prime$. One approach is to use $h^\prime(x) = \text{sign}(h(x))$, where $h(x) = w^T x + b$. Another approach is to consider $h^\prime(x, y) = y h(x)$, known as confidence margin. In this case, a correct classification by $h$ occurs when $h^\prime(x,y) > 0$, signifying that $x$ is classified correctly. Notably, the magnitude of $h(x)$ can be interpreted as the level of confidence in the prediction made by $h$.

\subsection{Similarities to SVM}
\label{appendix:similarities_to_svm}
Under similar assumptions, but taking the $L_2$ norm instead of the $L_\infty$ norm, that is, $\norm{x^\prime}_2 \leq u = r = 1$, $\norm{w}_2 \leq v = 1$ and assuming that $\xi = r = 1$, the same sample complexity guarantees work for,
\begin{mini*}|s|
    {w, t_1,\ldots, t_m}
    {\frac{1}{m} \sum_{i=1}^m t_i}
    {}{}
    \addConstraint{y_iw^Tx_i^\prime \ge \xi (2 - t_i)}
    \addConstraint{\ t_i \ge 0,\forall i \in 1, \ldots, m}
    \addConstraint{ \norm{w}_2 \leq 1},
\end{mini*}
which is equivalent to, 
\begin{mini*}|s|
    {w, t_1,\ldots, t_m}
    {\nu \norm{w}_2 + \frac{1}{m} \sum_{i=1}^m t_i}
    {}{}
    \addConstraint{y_iw^Tx_i^\prime \ge \xi (2 - t_i)}
    \addConstraint{\ t_i \ge 0,\forall i \in 1, \ldots, m},
\end{mini*}
with $\nu$ being a Lagrange variable. Note that this formulation is quite similar to the soft-margin version of the SVM,
\begin{mini*}|s|
    {w, b, t_1, \ldots, t_m}
    {\frac{1}{2} \norm{w}_2^2 + \lambda \sum_{i=1}^{m} t_i^p}
    {}{}
    \addConstraint{y_i(w^Tx_i + b) \geq 1 - t_i}
    \addConstraint{t_i > 0,\ \forall i \in 1 \ldots m}{}.
\end{mini*}

That being said, in another study \cite{Biggio2011}, the robustness of SVMs against adversarial data manipulation was examined. The authors considered a scenario where the adversary has control over training data and aims to tamper with the SVM learning procedure. They proposed a strategy based on kernel matrix correction to enhance the SVMs' robustness to such manipulation.
\section{Proofs of main results}
\label{appendix:proofs}

\subsection{Proof of Theorem \ref{theorem:main_theorem}}
\label{appendix:proofs:main_theorem}
Let us start by defining $\mathcal{H}^\prime = \{(x, y) \rightarrow yh(x),\ h \in \mathcal{H}\}$, and  $\mathcal{G} = \{\phi_{2,\zeta} \circ h^\prime,\ h^\prime \in \mathcal{H}^\prime\}$. By Theorem \ref{theorem:rademacher_bound}, we know that
\[
  \E{}{g(x)} \leq \Eemp{S}{g(s_i)} + 2\Rademacher{m}(\mathcal{G}) + \sqrt{\frac{\log{\frac{1}{\delta}}}{2m}},
\]
holds with probability at least $1-\delta$ for all $g \in \mathcal{G}$. This can be rewritten as, 
\begin{align*}
  \E{}{\phi_{2,\zeta}(y h(x))} &\leq \Eemp{S}{\phi_{2,\zeta}(y_ih(x_i))} + \\ 
  & 2\Rademacher{m}(\phi_{2,\zeta} \circ \mathcal{H}^\prime) + \sqrt{\frac{\log{\frac{1}{\delta}}}{2m}} \\
  \R{\mathrm{rob}}{h}{\zeta} &\leq \Eemp{S}{\phi_{2,\zeta}(y_ih(x_i))} + \\
  & 2\Rademacher{m}(\phi_{2,\zeta} \circ \mathcal{H}^\prime) + \sqrt{\frac{\log{\frac{1}{\delta}}}{2m}}.
\end{align*}
Through Talagrand's lemma, Lemma \ref{lemma:talagrand}, we can further upper
bound the Rademacher complexity and achieve,
\begin{equation}
   \R{\mathrm{rob}}{h}{\zeta} \leq \Eemp{S}{\phi_{2,\zeta}(y_ih(x_i))} + \frac{2}{\zeta}\Rademacher{m}(\mathcal{H}^\prime) + \sqrt{\frac{\log{\frac{1}{\delta}}}{2m}},
\end{equation}

with probability at least $1-\delta$, for a single, fixed $\zeta$ specified a-priori. 

Furthermore, note that the Rademacher complexity is in terms of $\mathcal{H}^\prime$, but we are interested in expressing it in terms of $\mathcal{H}$. To this end, it follows that the Rademacher complexity of $\mathcal{H}^\prime$ is equal to the Rademacher complexity of $\mathcal{H}$,
\begin{align*}
  \Rademacher{m}(\mathcal{H}^\prime) &= \E{\mathbb{P}^m}{
    \E{\mathbb{Q}^m}{
      \Opt{sup}
      {h^\prime \in \mathcal{H}^\prime}
      {\frac{1}{m} \sum_{i=1}^m \sigma_i h^\prime(s_i)}
      {}}} \\
  &=\E{\mathbb{P}^m}{
  \E{\mathbb{Q}^m}{
      \Opt{sup}
      {h \in \mathcal{H}}
      {\frac{1}{m} \sum_{i=1}^m \sigma_i y_ih(x_i)}
      {}}}\\
  &=\E{\mathbb{P}^m}{
  \E{\mathbb{Q}^m}{
      \Opt{sup}
      {h \in \mathcal{H}}
      {\frac{1}{m} \sum_{i=1}^m \sigma_i h(x_i)}
    {}}}\\
    &= \Rademacher{m}({\mathcal{H}}),
\end{align*}
where the third equality follows by the symmetry in $\sigma_i$ and because $y_i \in \{-1, 1\}$. 

As a consequence,
\begin{equation}
\label{equation:sample_complexit_bounded_by_rademacher}
  \R{\mathrm{rob}}{h}{\zeta} \leq \Eemp{S}{\phi_{2,\zeta}(y_ih(x_i))} + \frac{2}{\zeta}\Rademacher{m}(\mathcal{H}) + \sqrt{\frac{\log{\frac{1}{\delta}}}{2m}},
\end{equation}
holds, for all $h \in \mathcal{H}$, with probabiliy at least $1-\delta$, or, more precisely,
\begin{equation*}
  \mathbb{P}^m\Big(\R{\mathrm{rob}}{h}{\zeta}
  - \Eemp{S}{\phi_{2,\zeta}(y_ih(x_i))}
  \leq \epsilon + \frac{2}{\zeta}\Rademacher{m}(\mathcal{H}) \Big)
\end{equation*}
\begin{equation}
  \geq 1-\delta,
\end{equation}
which is equivalent to,
\begin{equation*}
  \mathbb{P}^m\Big(\R{\mathrm{rob}}{h}{\zeta}
  - \Eemp{S}{\phi_{2,\zeta}(y_ih(x_i))}
  > \epsilon + \frac{2}{\zeta}\Rademacher{m}(\mathcal{H}) \Big)
\end{equation*}
\begin{equation}
  \leq \delta,
\end{equation}

Notice that this holds for all $h \in \mathcal{H}$, and a given $\zeta$ specified beforehand. In particular it holds for $h$ that results in the supremum of the left hand side of the inequality.  Given that $\epsilon = \sqrt{\frac{\log{\frac{1}{\delta}}}{2m}}$, the right hand side be expressed just in terms of $\epsilon$,
\begin{equation*}
  \mathbb{P}^m\Big(
  \Opt{sup}
  {h \in \mathcal{H}}
  {\R{\mathrm{rob}}{h}{\zeta}}
  {}
  - \Eemp{S}{\phi_{2,\zeta}(y_ih(x_i))}
  > \epsilon + \frac{2}{\zeta}\Rademacher{m}(\mathcal{H}) \Big)
\end{equation*}
\begin{equation}
\label{equation:bound_exponential}
 \leq e^{-2m\epsilon^2}.
\end{equation}

In this part of the proof we generalize the result by showing that the bound holds uniformly for all $\zeta \in (0, r]$, with $r>0$, at the cost of an extra term in \eqref{equation:sample_complexit_bounded_by_rademacher}. 

Consider\footnote{This part follows closely the proof of Theorem 5.9 at \cite{Book_ML_Mohri2018}} now two sequences $\zeta_k$, $\epsilon_k$ with $\epsilon_k \in ]0,1]$. It follows that \eqref{equation:bound_exponential} holds for any fixed $k\geq1$:
\begin{equation*}
  \mathbb{P}^m\Big(
  \Opt{sup}
  {h \in \mathcal{H}}
  {\R{\mathrm{rob}}{h}{\zeta_k}}
  {}
  - \Eemp{S}{\phi_{2,\zeta_k}(y_ih(x_i))} 
\end{equation*}
\begin{equation*}
  > \epsilon_k + \frac{2}{\zeta_k}\Rademacher{m}(\mathcal{H}) \Big)  \leq e^{-2m\epsilon_k^2}.
\end{equation*}
However, the probability of the union (due to the supremum with respect to $k$) is bounded by the sum of the
probabilities, resulting in:
\begin{equation*}
  \mathbb{P}^m\Big(
  \Opt{sup}
  {k \geq 1, h \in \mathcal{H}}
  {\R{\mathrm{rob}}{h}{\zeta_k}}
  {}
  - \Eemp{S}{\phi_{2,\zeta_k}(y_ih(x_i))}
\end{equation*}
\begin{equation}
\label{equation:unif_convergence_zeta_epsilon}
  - \epsilon_k - \frac{2}{\zeta_k}\Rademacher{m}(\mathcal{H}) > 0 \Big) \leq \sum_{k \ge 1}e^{-2m\epsilon_k^2}.
\end{equation}
By choosing, 
\begin{equation}
    \label{equation:epsilon_k}
    \epsilon_k = \epsilon + \sqrt{\frac{\log k}{m}},
\end{equation}
the sum on the left-hand side of \eqref{equation:unif_convergence_zeta_epsilon} has an upper bound,
\begin{align*}
  \sum_{k \ge 1}e^{-2m\epsilon_k^2} &= \sum_{k \ge 1}e^{-2m(\epsilon + \sqrt{\frac{\log k}{m}})^2} \\
                              &\leq \sum_{k \ge 1}e^{-2m\epsilon^2} e^{-2\log k} \\
                              &= \sum_{k \ge 1}\frac{1}{k^2}e^{-2m\epsilon^2}\\
                              &\leq 2e^{-2m\epsilon^2}.
\end{align*}
Now, for $\gamma > 1$, choose $\zeta_k = \frac{r}{\gamma^k}$, and fix $\zeta_0 = r$. Then for any $\zeta \in ]0, r]$, $\exists k \ge 1$, s.t., $\zeta \in ]\zeta_k, \zeta_{k-1}]$. In addition, for this given $k$, $\zeta \leq \zeta_{k-1} = \gamma \zeta_k$, or, put differently, 
 \begin{equation}
     \label{equation:zeta_k_lowerbound}
     \zeta_k \geq \frac{\zeta}{\gamma}, 
 \end{equation}
 which means that $\zeta \leq \gamma \frac{r}{\gamma^k}$, and thus, 

 \begin{equation}
 \label{equation:log_k_upperbound}
     \sqrt{\log k} \leq \sqrt{\log \log_\gamma \frac{\gamma r}{\zeta}}.
 \end{equation}

Substituting \eqref{equation:epsilon_k}, \eqref{equation:zeta_k_lowerbound} and \eqref{equation:log_k_upperbound} into \eqref{equation:unif_convergence_zeta_epsilon} results in,
\begin{equation*}
  \mathbb{P}^m\Big(
  \Opt{sup}
  {\zeta \in ]0,r], h \in \mathcal{H}}
  {\R{\mathrm{rob}}{h}{\zeta}}
  {}
  - \Eemp{S}{\phi_{2,\zeta}(y_ih(x_i))}
\end{equation*}
\begin{equation}
  > \epsilon + \frac{2 \gamma}{\zeta}\Rademacher{m}(\mathcal{H}) + 
  \sqrt{\frac{\log \log_\gamma \frac{\gamma r}{\zeta}}{m}} \Big) \leq 2e^{-2m\epsilon^2},
\end{equation}

for any $r > 0$ and $\gamma >1 $. 

This is equivalent to,
\begin{equation*}
  \mathbb{P}^m\Big(
  \Opt{sup}
  {\zeta \in ]0,r], h \in \mathcal{H}}
  {\R{\mathrm{rob}}{h}{\zeta}}
  {}
  - \Eemp{S}{\phi_{2,\zeta}(y_ih(x_i))}
\end{equation*}
\begin{equation}
  \leq \sqrt{\frac{\log{\frac{2}{\delta}}}{2m}} + \frac{2 \gamma}{\zeta}\Rademacher{m}(\mathcal{H}) + \sqrt{\frac{\log \log_\gamma \frac{\gamma r}{\zeta}}{m}} \Big) \geq 1 - \delta,
\end{equation}

which concludes the proof.

\subsection{Proof of Lemma \ref{lemma:linear_classifier_rademacher_bound}}
\label{appendix:proofs:main_theorem:linear_classifier_rademacher_bound}
   \begin{align*}
    \EmpRademacher{S}(\mathcal{H}) &= \E{\mathbb{B}}{
      \Opt{sup}
      {h \in \mathcal{H}}
      {\frac{1}{m} \sum_{i=1}^m \sigma_i h(s_i)}
      {}} \\
    &= \frac{1}{m} \E{\mathbb{B}}{
      \Opt{sup}
      {h \in \mathcal{H}}
      {\sum_{i=1}^m \sigma_i w^T x_i^\prime}
      {\norm{w} \leq v}} \\
    &= \frac{1}{m} \E{\mathbb{B}}{
      \Opt{sup}
      {h \in \mathcal{H}}
      {w^T \sum_{i=1}^m \sigma_i x_i^\prime}
      {\norm{w} \leq v}} \\
    & (\text{by the definition of the dual norm}) \\
    &= \frac{1}{m} \E{\mathbb{B}}{
      v\norm{\sum_{i=1}^m \sigma_i x_i^\prime}_*} \\
    & (\text{by Jensen's inequality}) \\
    & \leq \frac{v}{m} \sqrt{\E{\mathbb{B}}{
      \norm{\sum_{i=1}^m \sigma_i x_i^\prime}_*^2}} \\
    & =  \frac{v}{m} \sqrt{\E{\mathbb{B}}{
      \norm{\sum_{i,j=1}^m \sigma_i \sigma_j x_i^\prime x_j^\prime}_*}} \\
    & (\text{because $\sigma_i$s are iid and follow a}\\
    & \text{symmetric Bernoulli distribution}) \\
    & =  \frac{v}{m} \sqrt{\E{\mathbb{B}}{
      \norm{\sum_{i}^m {x_i^\prime}^2}_*}}\\
    & \leq  \sqrt{\frac{v^2 u^2}{m}}
   \end{align*}

\subsection{Proof of Corollary \ref{corollary:linear_with_bounded}}
\label{appendix:proofs:main_theorem:linear_classifier_rademacher_bound:corollary}
  From theorem \ref{theorem:main_theorem} and lemma \ref{lemma:linear_classifier_rademacher_bound} we know that,
\begin{equation*}
  \R{\mathrm{rob}}{h}{\zeta} \leq \frac{1}{m} \sum_{i=1}^m \max \big(0, 2 - \frac{y_iw^Tx_i^\prime}{\zeta} \big) + \frac{2 \gamma}{\zeta}  \sqrt{\frac{v^2 u^2}{m}} 
\end{equation*}
\begin{equation*}
+ \sqrt{\frac{\log \log_\gamma \frac{\gamma r}{\zeta}}{m}} + \sqrt{\frac{\log{\frac{2}{\delta}}}{2m}},
\end{equation*}

holds with probability at least $1 - \delta$.

Notice that $h(x^\prime) = w^T x^\prime$ is continuous, and also
$\norm{h(x^\prime) - h(x_0^\prime)} \leq \norm{w^T} \norm{x^\prime -
  x_0^\prime}, \ \forall x, x_0 \in X$. In particular, given
$\xi > 0$, if $\norm{x - x_0} \leq \xi$, then
$\norm{h(x) - h(x_0)} \leq v \xi$. Taking $\zeta = v \xi$, concludes the proof.

\subsection{Proof of Corollary \ref{corollary:kernel}}
\label{appendix:proofs:lemma:kernel}
We know,  from (Theorem \ref{theorem:main_theorem}), that the following holds with probability at least $1-\delta$,

\begin{equation*}
    \R{\mathrm{rob}}{h}{\zeta} \leq \frac{1}{m} \sum_{i=1}^m \max \big(0, 2 - \frac{y_ih(x_i)}{\zeta} \big) + \frac{2 \gamma}{\zeta}\Rademacher{m}(\mathcal{H})
\end{equation*}
\begin{equation*}
 + \sqrt{\frac{\log \log_\gamma \frac{\gamma r}{\zeta}}{m}} + \sqrt{\frac{\log{\frac{2}{\delta}}}{2m}}.
\end{equation*}

The learners's goal is to solve the following minimization problem,
\begin{mini*}|s|
    {w \in \mathbb{H}}
    {\frac{1}{m} \sum_{i=1}^m \max \big(0, 2 - \frac{y_iw^T\psi(x_i)}{\zeta} \big)}
    {}{}
    \addConstraint{\norm{w}_\mathbb{H} \leq v},
\end{mini*}
but note that this problem can be written as follows by introducing a Lagrange variable $\lambda$,
\begin{mini*}|s|
    {w \in \mathbb{H}}
    {\frac{1}{m} \sum_{i=1}^m \max \big(0, 2 - \frac{y_iw^T\psi(x_i)}{\zeta} \big) - \lambda \norm{w}_\mathbb{H}}
    {}{}.
\end{mini*}

This form is of particular interest because, given the representer theorem (Theorem \ref{theorem:representer}), we know that the solution of this optimization has the form $w = \sum_{i=1}^{m} \alpha_i \psi(x_i)$, and hence, instead of solving the former optimization problem, in the Hilbert space $\mathbb{H}$, we can solve the following in the original space,
\begin{mini*}|s|
    {\alpha \in X^m}
    {\Biggl\{  \frac{1}{m} \sum_{i=1}^m \max \big(0, 2 - \frac{y_i }{\zeta} \sum_{j=1}^m \alpha_j \psi(x_j) \psi(x_i) \big) \\
    - \lambda \sqrt{\sum_{i,h=1}^m \alpha_i \alpha_j \psi(x_i) \psi(x_j)} \Biggl\}}
    {}{}
\end{mini*}
\begin{mini*}|s|
    {\alpha \in X^m}
    {\Biggl\{  \frac{1}{m} \sum_{i=1}^m \max \big(0, 2 - \frac{y_i }{\zeta} \sum_{j=1}^m \alpha_j K(x_i, x_j)\big) \\
    - \lambda \sqrt{\sum_{i,h=1}^m \alpha_i \alpha_j K(x_i, x_j)}   \Biggl\}}
    {}{}
\end{mini*}
\begin{mini*}|s|
    {\alpha \in X^m}
    {\Biggl\{  \frac{1}{m} \sum_{i=1}^m \max \big(0, 2 - \frac{y_i (\mathbf{K} \alpha)_i}{\zeta} \big) - \lambda \sqrt{ \alpha^T \mathbf{K} \alpha}  \Biggl\}}
    {}{}
\end{mini*}
\begin{mini*}|s|
    {\alpha \in X^m}
    {\frac{1}{m} \sum_{i=1}^m \max \big(0, 2 - \frac{y_i (\mathbf{K} \alpha)_i}{\zeta} \big)}
    {}{}
    \addConstraint{\alpha^T \mathbf{K} \alpha \leq v^2},
\end{mini*}
where $\mathbf{K} = [K(x_i, x_j)]_{ij}$ is the Gramiam matrix.

We start by first bounding the Rademacher complexity of the hypothesis class $\mathcal{H}$. This proof is very similar to the proof of \ref{lemma:linear_classifier_rademacher_bound}:
   \begin{align*}
    \EmpRademacher{S}(\mathcal{H}) &= \E{\mathbb{B}}{
      \Opt{sup}
      {h \in \mathcal{H}}
      {\frac{1}{m} \sum_{i=1}^m \sigma_i h(x_i)}
      {}} \\
    &= \frac{1}{m} \E{\mathbb{B}}{
      \Opt{sup}
      {h \in \mathcal{H}}
      {\sum_{i=1}^m \sigma_i w^T \psi(x_i)}
      {\norm{w}_\mathbb{H} \leq v}} \\
    &= \frac{1}{m} \E{\mathbb{B}}{
      \Opt{sup}
      {h \in \mathcal{H}}
      {w^T \sum_{i=1}^m \sigma_i \psi(x_i)}
      {\norm{w}_\mathbb{H} \leq v}} \\
      & (\text{by the definition of the dual norm}) \\
    &= \frac{1}{m} \E{\mathbb{B}}{
      v\norm{\sum_{i=1}^m \sigma_i \psi(x_i)}_{\mathbb{H}_*}} \\
      & (\text{Jensen's inequality}) \\
    & \leq \frac{v}{m} \sqrt{\E{\mathbb{B}}{
      \norm{\sum_{i=1}^m \sigma_i \psi(x_i)}_{\mathbb{H}_*}^2}} \\
    & =  \frac{v}{m} \sqrt{\E{\mathbb{B}}{
      \norm{\sum_{i,j=1}^m \sigma_i \sigma_j \psi(x_i) \psi(x_j)}_{\mathbb{H}_*}}} \\
      & (\text{because $\sigma$ is iid}) \\
    & =  \frac{v}{m} \sqrt{\E{\mathbb{B}}{
      \norm{\sum_{i}^m {\psi(x_i)}^2}_{\mathbb{H}_*}}} \\
      & (\text{because $\sigma$ is iid}) \\
    & =  \frac{v}{m} \sqrt{\E{\mathbb{B}}{
      \norm{\sum_{i}^m {K(x_i,x_i)}}_{\mathbb{H}_*}}} \\
    & \leq  \sqrt{\frac{v^2 u^2}{m}}.
   \end{align*}

This means, (Theorem \ref{theorem:main_theorem}), that the following holds with probability at least $1-\delta$,
\begin{equation*}
    \R{\mathrm{rob}}{h}{\zeta} \leq \frac{1}{m} \sum_{i=1}^m 
    \max \big(0, 2 - \frac{y_i (\mathbf{K} \alpha)_i}{\zeta} \big) + \frac{2 \gamma}{\zeta}\sqrt{\frac{v^2 u^2}{m}} +
\end{equation*}
\begin{equation*}
 \sqrt{\frac{\log \log_\gamma \frac{\gamma r}{\zeta}}{m}} + \sqrt{\frac{\log{\frac{2}{\delta}}}{2m}}.
\end{equation*}

\subsection{Proof of Theorem \ref{theorem:main_theorem_multiclass}}
\label{appendix:proofs:theorem:main_theorem_multiclass}
Let us first define $\mathcal{H}^\prime = \{(x, y) \rightarrow h(x,y) - \Opt{max}{y^\prime \neq y}{h(x, y ^\prime)}{},\ h \in \mathcal{H}\}$, and $\mathcal{G} = \{\phi_{2,\zeta} \circ h^\prime,\ h^\prime \in \mathcal{H}^\prime\}$, where $\phi_{2,\zeta}$ is a surrogate loss function.

We know, that,
\begin{equation*}
    \R{\mathrm{rob}}{h}{\zeta} \leq \Eemp{S}{\phi_{2,\zeta}(h(x_i,y_i) - \Opt{max}{y^\prime \neq y_i}{h(x_i, y ^\prime)}{} )} 
\end{equation*}
\begin{equation*}
+ \frac{2}{\zeta}\Rademacher{m}(\mathcal{H}^\prime) + \sqrt{\frac{\log{\frac{1}{\delta}}}{2m}},
\end{equation*}
and the $\Rademacher{m}(\mathcal{H}^\prime)$ is given by,
\begin{align*}
  \Rademacher{m}(\mathcal{H}^\prime) &= \E{\mathbb{P}^m}{
    \E{\mathbb{Q}^m}{
      \Opt{sup}
      {h^\prime \in \mathcal{H}^\prime}
      {\frac{1}{m} \sum_{i=1}^m \sigma_i h^\prime(s_i)}
      {}}} \\
  &=\mathbb{E}_{\mathbb{P}^m}\big[
        \mathbb{E}_{\mathbb{Q}^m}\big[
        \underset{\substack{h \in \mathcal{H}}}{\mathrm{sup}}
      \frac{1}{m} \sum_{i=1}^m \sigma_i h(x_i,y_i) \\
      & \quad \quad \quad \quad \quad \quad \quad \quad - \Opt{max}{y^\prime \neq y_i}{h(x_i, y ^\prime)}{}
      \big]
      \big]\\
  &=\E{\mathbb{P}^m}{
  \E{\mathbb{Q}^m}{
      \Opt{sup}
      {h \in \mathcal{H}}
      {\frac{1}{m} \sum_{i=1}^m \sigma_i h(x_i,y_i)}
      {}}} \\
  & \quad + \E{\mathbb{P}^m}{
  \E{\mathbb{Q}^m}{
      \Opt{sup}
      {h \in \mathcal{H}}
      {\frac{1}{m} \sum_{i=1}^m -\sigma_i \Opt{max}{y^\prime \neq y_i}{h(x_i, y ^\prime)}{}}
      {}}}\\ 
  &=\E{\mathbb{P}^m}{
  \E{\mathbb{Q}^m}{
      \Opt{sup}
      {h \in \mathcal{H}}
      {\frac{1}{m} \sum_{i=1}^m \sigma_i h(x_i,y_i)}
      {}}} \\
  & \quad + \E{\mathbb{P}^m}{
  \E{\mathbb{Q}^m}{
      \Opt{sup}
      {h \in \mathcal{H}}
      {\frac{1}{m} \sum_{i=1}^m \sigma_i \Opt{max}{y^\prime \neq y_i}{h(x_i, y ^\prime)}{}}
      {}}}\\ 
  &= \mathbb{E}_{\mathbb{P}^m}\big[
        \mathbb{E}_{\mathbb{Q}^m}\big[
            \underset{\substack{h \in \mathcal{H}}}{\mathrm{sup}}
            \frac{1}{m} \sum_{i=1}^m \sigma_i (h(x_i,y) \mathds{1}_{y=y_i} \\
            & \quad \quad \quad \quad \quad \quad \quad \quad + \Opt{max}{y}{h(x_i, y) \mathds{1}_{y \neq y_i}}{})
      \big]
    \big] \\  
  & \leq \E{\mathbb{P}^m}{
  \E{\mathbb{Q}^m}{
      \Opt{sup}
      {h \in \mathcal{H}}
      {\frac{1}{m} \sum_{i=1}^m \sigma_i k |h(x_i,y_i)| }
      {}}} \\
  & \text{(by Talagrand's lemma)}\\
  & \leq k \E{\mathbb{P}^m}{
  \E{\mathbb{Q}^m}{
      \Opt{sup}
      {h \in \mathcal{H}}
      {\frac{1}{m} \sum_{i=1}^m \sigma_i h(x_i,y_i) }
      {}}}\\
  & =  k \Rademacher{m}(\mathcal{H})
\end{align*}

And the bound follows by performing the same steps of the proof of Theorem \ref{theorem:main_theorem}.
\section{Numerical examples}
\label{appendix:numerical_examples}

The table below summarizes the number of training and test samples for various cases, impacting the reported empirical accuracy levels in the manuscript.
\begin{small}
\begin{center}
\begin{tabular}{lrr}
    Data & Training samples & Test samples\\
     \hline
    NIST 0/1 & 12665 & 2115\\
    NIST 3/8 & 11982 & 1984\\
    CIFAR Cat/Dog & 10000 & 2000\\
    CIFAR Dog/Airplane & 10000 & 2000
\end{tabular}
\end{center}
\end{small}
The code relies heavily on PyTorch and is publicly available at: 
\begin{itemize}
    \item[] \url{https://github.com/f2cf2e10/advML}
\end{itemize}

For FGSM and PGD training/attacks we used SGD as optimization procedure with batch size 100, 10 epochs and seed torch seed set to 171. For our proposed approach we use a convex optimization solver for LP.

\subsection{MNIST/CIFAR10 attacks}

In this study, we examine the impact of an FGSM adversary on clean images, using examples of the number 3 and number 8 from the NIST dataset and cats and dogs from the CIFAR10 dataset.
\begin{figure}[h!]
  \centering
  \includegraphics[width=0.15\textwidth]{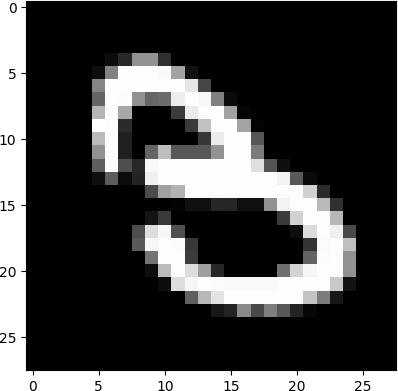}
  \includegraphics[width=0.15\textwidth]{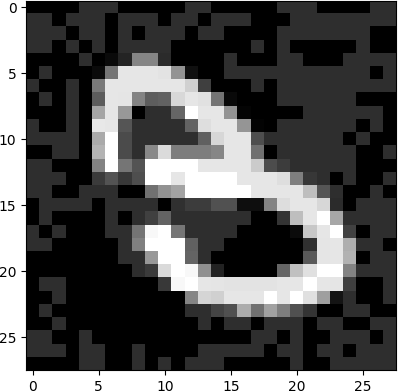}
  \includegraphics[width=0.15\textwidth]{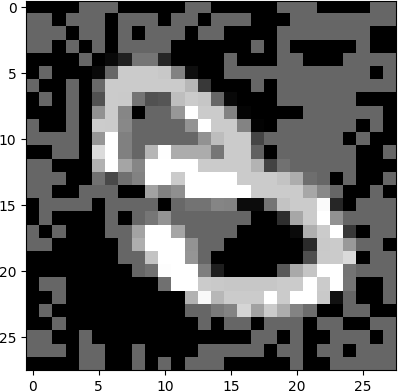}
  \caption{NIST 3:  clean (left), tampered $\xi=0.1$ (middle) and tampered $\xi=0.25$ (right)}
\end{figure}

\begin{figure}[ht!]
  \centering
  \includegraphics[width=0.15\textwidth]{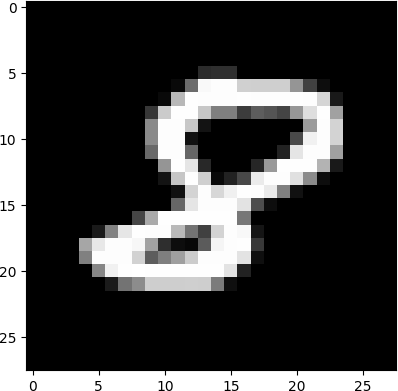}
  \includegraphics[width=0.15\textwidth]{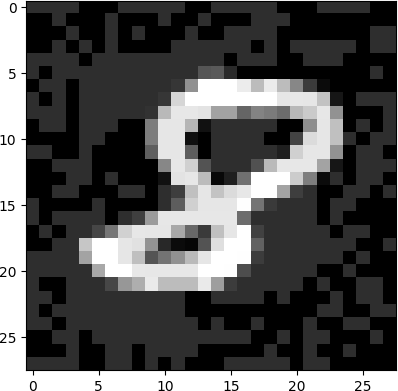}
  \includegraphics[width=0.15\textwidth]{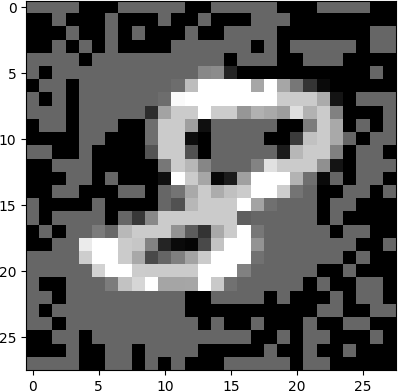}
  \caption{NIST 8:  clean (left), tampered $\xi=0.1$ (middle) and tampered $\xi=0.25$ (right)}
\end{figure}

 We compare the original clean images with their corresponding adversarial versions at two power levels, $\xi = 0.1$ and $\xi = 0.25$. The purpose of this visual analysis is to illustrate and explore how the FGSM perturbations alter the visual appearance of the images and discuss the implications of such attacks on image recognition systems.

\begin{figure}[h!]
  \centering
  \includegraphics[width=0.15\textwidth]{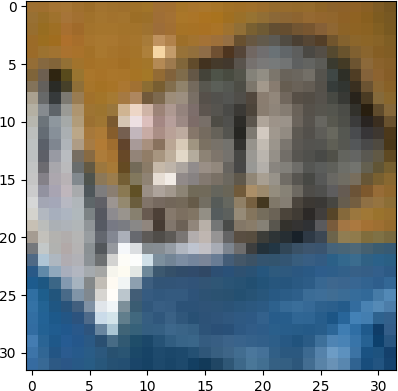}
  \includegraphics[width=0.15\textwidth]{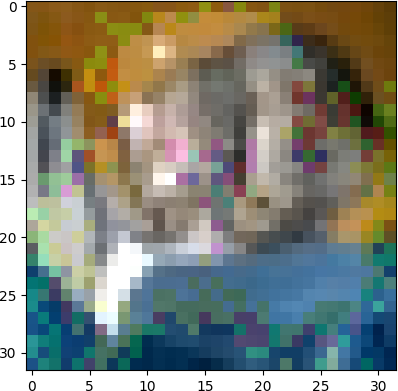}
  \includegraphics[width=0.15\textwidth]{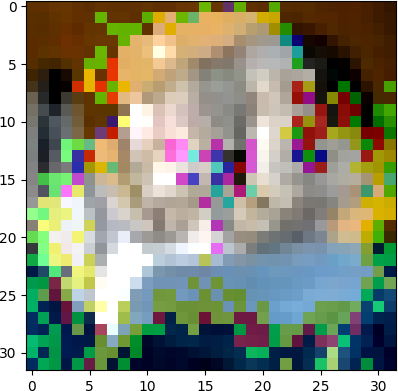}
  \caption{CIFAR10 cat:  clean (left), tampered $\xi=0.1$ (middle) and tampered $\xi=0.25$ (right)}
\end{figure}

\begin{figure}[h!]
  \centering
  \includegraphics[width=0.15\textwidth]{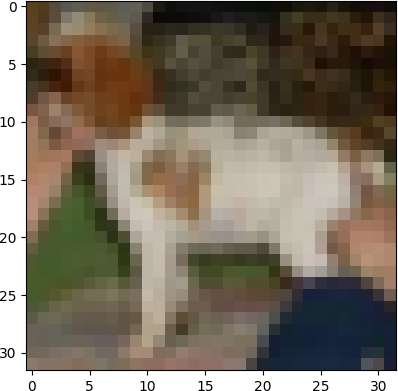}
  \includegraphics[width=0.15\textwidth]{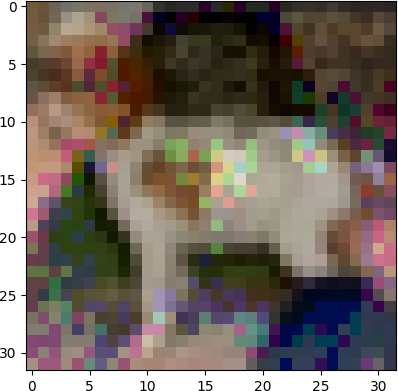}
  \includegraphics[width=0.15\textwidth]{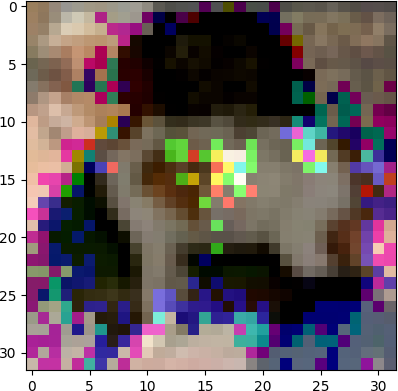}
  \caption{CIFAR10 dog:  clean (left), tampered $\xi=0.1$ (middle) and tampered $\xi=0.25$ (right)}
\end{figure}

\end{document}